\documentclass[12pt]{article}
\usepackage[utf8]{inputenc}
\usepackage[margin=1in]{geometry}
\usepackage{amsmath, amsfonts, amssymb, bm, amsthm}
\usepackage{hyperref}  
\usepackage{cleveref}
\usepackage[
  backend=biber,
  style=numeric,
  citestyle=numeric
]{biblatex}
\usepackage{booktabs}
\usepackage{graphicx}
\usepackage{tikz}
\usetikzlibrary{positioning, arrows.meta, shapes}
\usepackage[table]{xcolor}
\usepackage{algorithm2e}
\usepackage{algpseudocode}
\usepackage{listings}
\hypersetup{colorlinks=true, linkcolor=blue, citecolor=blue, urlcolor=blue}

\usetikzlibrary{shapes.geometric,arrows.meta,positioning,calc}

\newtheorem{lemma}{Lemma}
\newtheorem{proposition}{Proposition}

\newcommand{\Ncal}{\mathcal{N}}
\newcommand{\KL}{D_{\mathrm{KL}}}
\newcommand{\Cat}{\mathrm{Cat}}
\newcommand{\softplus}{\mathrm{softplus}}
\newcommand{\diag}{\mathrm{diag}}
\newcommand{\R}{\mathbb{R}}
\newcommand{\E}{\mathbb{E}}

\title{\textbf{eXact-Prior Variational Autoencoder (X-VAE): 
Learning Data-Adaptive Gaussian Mixture Priors for Latent Distributions}}

\author{Qijun Chen \textsuperscript{1}, Shaofan Li \textsuperscript{1,$\dagger$} \\
\\
	\textsuperscript{1} Department of Civil and Environmental Engineering, \\
	University of California, Berkeley, CA, 94720, USA\\
	\textsuperscript{$\dagger$} Correspondence: \texttt{shaofan@berkeley.edu} \\
}

\date{\today}

\begin{document}
\maketitle

\begin{abstract}
Variational Autoencoders (VAEs) commonly assume a standard isotropic Gaussian prior over the latent space, an assumption that often fails to capture the true distribution of latent representations for complex datasets. This mismatch can limit reconstruction accuracy, reduce sample quality, and constrain the expressive power of the learned latent space. We propose the eXact-Prior Variational Autoencoder (X-VAE), a framework that replaces the conventional standard normal prior with a Gaussian prior derived from the latent representations of a pretrained autoencoder (AE). Specifically, the empirical mean and standard deviation of the AE latent codes are used to parameterize a data-adaptive prior that more closely reflects the underlying structure of the training data. During generation, X-VAE introduces a latent scaling factor that enables explicit control over the variance of the sampled latent vectors, providing a simple mechanism for balancing sample diversity and fidelity. This flexibility makes the proposed approach particularly well suited for applications such as industrial and engineering design, where generated solutions must satisfy strict structural or functional constraints while still permitting meaningful design exploration. We present the mathematical formulation of X-VAE, derive the corresponding KL divergence objective for the proposed prior, and evaluate the method on standard benchmark datasets. Experimental results demonstrate that X-VAE preserves reconstruction quality while producing latent representations that better align with the empirical data distribution, leading to improved controllability and more realistic generated samples.
\end{abstract}
\section{Introduction}
\label{sec:intro}

Variational Autoencoders (VAEs) have become one of the foundational frameworks in deep generative modeling by combining probabilistic inference with deep neural representation learning to model complex data distributions~\cite{kingma2022vae,rezende2014}. A key assumption underlying the standard VAE is the use of a fixed isotropic Gaussian prior over the latent variables,

\begin{equation}
p(z)=\mathcal{N}(0,\mathbf{I}),
\end{equation}

which simplifies the computation of the Kullback--Leibler (KL) divergence within the Evidence Lower Bound (ELBO) while encouraging a smooth and regularized latent space. Although mathematically convenient, this assumption is often inadequate for representing the latent structure of complex datasets. In many real-world applications, particularly those involving multimodal, high-dimensional, or structurally constrained data, the aggregated posterior deviates substantially from a standard Gaussian distribution~\cite{dai2019diagnosing,arvanitidis2017latent}. Consequently, the mismatch between the assumed prior and the learned latent distribution can reduce reconstruction quality, limit the expressiveness of the latent space, and contribute to phenomena such as prior--posterior mismatch during generation and posterior collapse~\cite{rezende2016flow,lucas2019understanding,he2019lagging}.

To alleviate these limitations, numerous extensions to the VAE framework have proposed more expressive latent priors, including mixture-based priors such as VampPrior~\cite{tomczak2018vamp}, flow-based priors~\cite{kingma2017improvingvariationalinferenceinverse}, and hierarchical latent-variable models~\cite{sonderby2016ladder,vahdat2021Hierarchy}. While these approaches improve the flexibility of the latent distribution, they generally introduce additional trainable parameters, increased computational complexity, and more challenging optimization procedures.

In contrast, deterministic Autoencoders (AEs) learn latent representations by directly minimizing reconstruction error without imposing a probabilistic prior~\cite{goodfellow2016deep}. Although conventional autoencoders are not generative models, the latent embeddings they learn capture important structural characteristics of the underlying data~\cite{ghosh2019from}. The empirical statistics of these latent representations therefore provide a simple yet informative estimate of the latent distribution and motivate the use of a data-adaptive prior for variational learning.

Motivated by this observation, we propose the \textbf{eXact-Prior Variational Autoencoder (X-VAE)}, a framework that replaces the conventional isotropic Gaussian prior with a Gaussian distribution estimated from the latent representations of a pretrained autoencoder. Let

\[
\left\{z_i^{AE}\right\}_{i=1}^{N}
\]

denote the latent embeddings produced by the pretrained autoencoder for the training dataset. The empirical mean and covariance of these embeddings are computed as

\begin{equation}
\mu_{AE}
=
\frac{1}{N}
\sum_{i=1}^{N}
z_i^{AE},
\end{equation}

\begin{equation}
\Sigma_{AE}
=
\frac{1}{N-1}
\sum_{i=1}^{N}
\left(z_i^{AE}-\mu_{AE}\right)
\left(z_i^{AE}-\mu_{AE}\right)^T.
\end{equation}

These empirical statistics define the latent prior

\begin{equation}
p_{AE}(z)
=
\mathcal{N}
\left(
\mu_{AE},
\Sigma_{AE}
\right),
\end{equation}

which replaces the standard normal prior in the VAE objective. By grounding the latent prior in the empirical distribution learned by the autoencoder, X-VAE encourages improved alignment between the prior and the intrinsic structure of the data while preserving the simplicity and analytical tractability of Gaussian latent variables.

To further improve generative flexibility, X-VAE introduces a global latent scaling parameter, $\alpha$, that is applied only during the sampling stage. Generated latent vectors are sampled according to

\begin{equation}
z_{\mathrm{gen}}
\sim
\mathcal{N}
\left(
\mu_{AE},
\alpha^{2}\Sigma_{AE}
\right),
\end{equation}

where $\alpha$ scales the standard deviation of the latent distribution. Smaller values of $\alpha$ generate samples that remain close to the empirical latent manifold, producing higher-fidelity reconstructions, whereas larger values increase latent variability and encourage exploration of less frequently observed regions of the latent space. This provides a simple yet effective mechanism for controlling the trade-off between generation diversity and sample fidelity without requiring retraining.

Unlike many existing learned-prior approaches that jointly optimize the prior together with the encoder and decoder, X-VAE decouples prior estimation from variational learning. The latent prior is estimated once from a pretrained deterministic autoencoder and subsequently treated as a fixed probabilistic prior during VAE training. This separation preserves the simplicity and stability of the original VAE optimization procedure while allowing the prior to more accurately reflect the empirical latent structure of the training data. Consequently, X-VAE retains the computational efficiency of conventional VAEs without introducing additional trainable prior networks, hierarchical latent models, or flow-based transformations.

Although X-VAE is evaluated on standard generative benchmarks, including MNIST and CelebA~\cite{liu2015faceattributes}, its motivation extends well beyond conventional image synthesis. Many engineering and industrial design problems require generated samples to satisfy strict geometric, structural, or physical constraints while simultaneously maintaining sufficient diversity for design exploration. By constructing the latent prior directly from data-derived representations, X-VAE provides a practical framework for constrained generative modeling in applications such as vehicle crash deformation analysis, topology optimization, and structural ship hull generation~\cite{LI2022100849,tran2024car}.

The primary contributions of this work are summarized as follows:

\begin{itemize}

\item We propose the \textbf{eXact-Prior Variational Autoencoder (X-VAE)}, which replaces the conventional isotropic Gaussian prior with a data-adaptive Gaussian prior estimated from the latent representations of a pretrained autoencoder.

\item We derive the empirical latent prior from the mean and covariance of pretrained autoencoder embeddings, preserving an analytically tractable KL divergence while improving alignment between the latent prior and the underlying data distribution.

\item We introduce a latent variance scaling mechanism that provides explicit control over the diversity--fidelity trade-off during generation without requiring additional model training.

\item We demonstrate the effectiveness of X-VAE on standard image generation benchmarks and highlight its applicability to engineering and industrial design problems requiring structurally constrained yet diverse generative models.

\end{itemize}

\section{Methodology}
\label{sec:method}

In a standard Variational Autoencoders (VAEs), the latent prior is typically chosen as a fixed isotropic Gaussian, such as $p(z) = \mathcal{N}(0, I)$ while the posterior $q(z|x)$ is learned from $q(z|x) = \mathcal{N}(\mu(x), diag(\Sigma^2(x)))$, the KL divergence term in the loss forces $q(z|x)$ to be close to $p(z)$, a standard normal distribution.
This choice simplifies the Kullback-Leibler (KL) divergence term in the evidence lower bound (ELBO)~\cite{kingma2022vae,hoffman2016elbo} and provides a regularized latent space, however, it may be overly restrictive.
The learned posterior $q_\phi(z|x)$ often deviates from the fixed prior distribution or the true prior distribution is mostly not in a standard normal distribution~\cite{rosca2018distribution,makhzani2015adversarial,koike2022autovae}.
This issue leads to a mismatch between the training and generation distributions.
In contrast, for pure autoencoders, there is only reconstruction by using match errors such as mean square errors (MSE) or Binary Cross-Entropy (BCE)~\cite{vincent2010stacked, rumelhart1986learning}.
The primary goal for an autoencoder is to use its latent space to force a compressed representation of the data~\cite{hinton2006reducing}.
The parameters such as mean ($\mu$), standard deviation ($\sigma$), covariance matrix ($\Sigma$), and weights of gaussian mixture model ($\pi$) learned by the latent space of an Autoencoder if fitted to a mixture of Gaussian could be translated as nonlinear transformation of the data manifold~\cite{arvanitidis2017latent}.
Based on that we can treat the compressed manifold as an estimated true data.
We can approach implicit latent statistics in two assumptions. In the first assumption, we assume the latent space follows a single Gaussian while the second we assume that it follows a Gaussian of Mixture.

By the first approach, which is simple and straightforward, the VAE model learns an expressive and data-driven statistics of an implicit prior learned by Autoencoder from data rather than fixing it and the latent space geometry aligns with dataset's manifold from Autoencoder so the prior becomes data-aware. 
This is basically replacing a naive isotropic prior $\Ncal(0,I)$ with a nonlinear learned prior. 
Concretely, let \(\{x_i\}_{i=1}^N\) be the training data, and let \(z_i = E_{\mathrm{AE}}(x_i)\) be the corresponding latent codes produced by a pretrained AE encoder ($E_{AE}$).
We compute  
\begin{equation}
\label{equ:single_gaussian}
    \mu_{\mathrm{AE}} = \frac{1}{N} \sum_{i=1}^{N} z_i, \quad 
    \Sigma_{\mathrm{AE}} = \frac{1}{N} \sum_{i=1}^{N}(z_i - \mu_{\mathrm{AE}})(z_i-\mu_{\mathrm{AE}})^T\\
\end{equation}
so the prior can be defined as
\begin{equation}  
    p_{\mathrm{AE}}(z) = \mathcal{N}\big(z \,; \, \mu_{\mathrm{AE}}, \, \mathrm{diag}(\sigma_{\mathrm{AE}}^2)\big)
\end{equation}
with the assumption that the approximate posterior is a Gaussian distribution, we have
\begin{equation}
    q_\phi(z|x) = \mathcal{N}\bigl(z; \mu_\phi(x), \,\mathrm{diag}(\sigma_\phi(x)^2)\bigr)
\end{equation}
instead of the $\Ncal(0, I)$ in standard VAE.

For the second approach, still let $z_i=E_{\mathrm{AE}}(x_i)$ the latent codes of a pretrained AE encoder with a $K$-dimensional latent instead.
We fit a $K$-component diagonal Gaussian mixture on the AE latent codes $\{z_i\}$ (e.g.\ by expectation--maximization) and freeze it as the prior, so the latent space inherently aligns with the AE's cluster structure, mathematically,
\begin{equation}
    p_{AE}(z)=\sum_{k=1}^{K}\pi^p_k\,\Ncal\!\big(z;\,M^p_{k,:},\,\diag(S^p_{k,:}{}^{2})\big),
\label{eq:gmm-prior}
\end{equation}
where $M^p,S^p\in\R^{K\times K}$ collect the component means and standard deviations and $K$ is both the number of clusters and the latent dimensionality.
Obtaining mixture weights $\pi^p\in\Delta^{K-1}$, component means $M^p_{k,:}\in\R^{K}$, and standard deviations $S^p_{k,:}\in\R^{K}$, $k=1,\dots,K$, these define the prior~\Cref{eq:gmm-prior}, which is held fixed throughout VAE training.
Unlike learned priors~\cite{tomczak2018vamp,kingma2017improvingvariationalinferenceinverse,vahdat2021Hierarchy}, \Cref{eq:gmm-prior} introduces no trainable prior parameters, but it is estimated once and frozen, anchoring the VAE latent space to the AE's manifold and alleviating prior--posterior mismatch.
The VAE encoder emits, per component $k$, an assignment logit $\ell_k$, a posterior mean $\mu^q_k$, and a log-variance $\log\sigma^{q2}_k$ (so $\sigma^q_k=\exp(\tfrac12\log\sigma^{q2}_k)$).
The posterior mixture weights (responsibilities) are aggregated at the batch level over $N$ samples,
\begin{equation}
    \pi^q_k=\frac{1}{N}\sum_{i=1}^{N}\frac{\softplus(\ell_{i,k})}{\sum_{j=1}^{K}\softplus(\ell_{i,j})},
\label{eq:piq}
\end{equation}
where $\softplus$ keeps the sum and non-negative so that $\pi^q\in\Delta^{K-1}$.
The approximate posterior for coordinate $k$ is the Gaussian $\Ncal^q_k=\Ncal(\mu^q_k,\sigma^{q2}_k)$.
The latent code is built one coordinate at a time by transporting noise---sourced from a categorically routed prior component---onto the encoder's posterior Gaussian as illustrated in \Cref{fig:arch}.

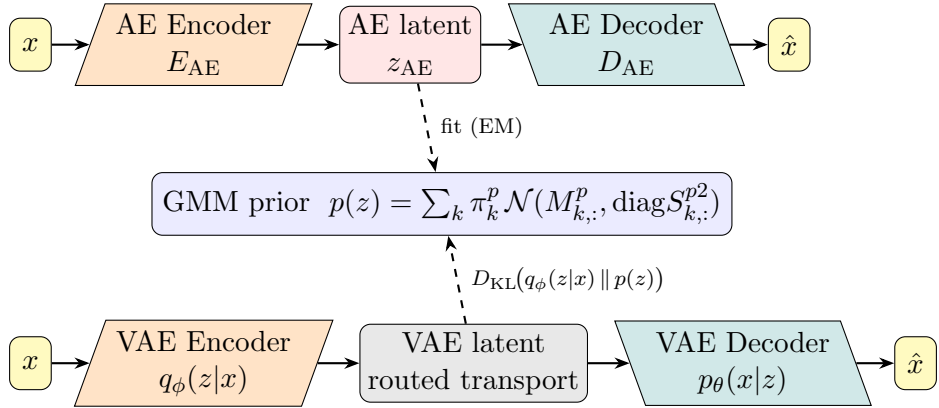
\begin{figure}[!htb]
\centering
\begin{tikzpicture}[font=\small,
  enc/.style={trapezium, draw, trapezium left angle=70, trapezium right angle=110,
              minimum height=1.05cm, minimum width=1.0cm, fill=orange!22, align=center},
  dec/.style={trapezium, draw, trapezium left angle=110, trapezium right angle=70,
              minimum height=1.05cm, minimum width=1.0cm, fill=teal!18, align=center},
  io/.style={draw, rounded corners, minimum height=0.7cm, minimum width=0.55cm, fill=yellow!30},
  lat/.style={draw, rounded corners, fill=red!10, align=center, inner sep=3pt},
  vlat/.style={draw, rounded corners, fill=gray!18, align=center, inner sep=3pt},
  prior/.style={draw, rounded corners, fill=blue!8, align=center, inner sep=4pt},
  ar/.style={-{Stealth[length=2mm]}, thick},
  dar/.style={-{Stealth[length=2mm]}, thick, dashed}]

\node[io] (x1) {$x$};
\node[enc, right=0.5cm of x1] (ae) {AE Encoder\\$E_{\mathrm{AE}}$};
\node[lat, right=0.55cm of ae] (zae) {AE latent\\$z_{\mathrm{AE}}$};
\node[dec, right=0.55cm of zae] (aed) {AE Decoder\\$D_{\mathrm{AE}}$};
\node[io, right=0.5cm of aed] (xh1) {$\hat{x}$};
\draw[ar] (x1)--(ae); \draw[ar] (ae)--(zae); \draw[ar] (zae)--(aed); \draw[ar] (aed)--(xh1);

\node[io, below=3.5cm of x1] (x2) {$x$};
\node[enc, right=0.5cm of x2] (ve) {VAE Encoder\\$q_\phi(z|x)$};
\node[vlat, right=0.55cm of ve] (zq) {VAE latent\\routed transport};
\node[dec, right=0.55cm of zq] (vd) {VAE Decoder\\$p_\theta(x|z)$};
\node[io, right=0.5cm of vd] (xh2) {$\hat{x}$};
\draw[ar] (x2)--(ve); \draw[ar] (ve)--(zq); \draw[ar] (zq)--(vd); \draw[ar] (vd)--(xh2);

\node[prior] (pr) at ($(zae)!0.5!(zq)$)
  {GMM prior \ $p(z)=\sum_{k}\pi^p_k\,\Ncal(M^p_{k,:},\diag S^{p2}_{k,:})$};
\draw[dar] (zae) -- (pr) node[midway, right=1pt] {\scriptsize fit (EM)};
\draw[dar] (zq) -- (pr) node[midway, right=1pt] {\scriptsize $\KL\!\big(q_\phi(z|x)\,\|\,p(z)\big)$};
\end{tikzpicture}
\caption{Architecture of the proposed X-VAE. A pretrained autoencoder learns latent codes
$z_{\mathrm{AE}}$; a $K$-component diagonal Gaussian mixture fit on these codes defines a fixed,
data-driven prior $p(z)$. The VAE encoder produces a per-coordinate posterior $q_\phi(z|x)$ and
mixture responsibilities; each latent coordinate is formed by routed transport~\Cref{eq:route}--%
\Cref{eq:code} and regularized toward $p(z)$ by the mixture KL~\Cref{eq:kl}.}
\label{fig:arch}
\end{figure}

\paragraph{Transport operator.}
For coordinate $k$, the transport map sends a scalar noise sample $\epsilon$ to $\Ncal^q_k$ using the
coordinate's diagonal prior as reference
\begin{equation}
T_k(\epsilon)=\mu^q_k+\frac{\sigma^q_k}{\sigma^p_k}\,(\epsilon-\mu^p_k).
\label{eq:transport}
\end{equation}
If $\epsilon\sim\Ncal(\mu^p_k,\sigma^{p2}_k)$ then $T_k(\epsilon)\sim\Ncal^q_k$, i.e.\
\Cref{eq:transport} is the reparameterization of the posterior Gaussian (~\Cref{lem:transport})~\cite{Marzouk_2016}.

\paragraph{Categorical routing.}
A prior component is selected per coordinate by a categorical over mixture weights
$\boldsymbol{\pi}\in\{\pi^p,\pi^q\}$, and the noise is drawn from that component's $k$-th coordinate,
\begin{equation}
c_k\sim\Cat(\boldsymbol{\pi}),\qquad \epsilon_k\sim\Ncal\!\big(M^p_{c_k,k},\,S^{p2}_{c_k,k}\big).
\label{eq:route}
\end{equation}

\paragraph{Latent code.}
Each coordinate is then either \emph{transported} or \emph{emitted directly}, and the code collects
the $K$ coordinates,
\begin{equation}
z_k=\begin{cases}T_k(\epsilon_k)&\text{(transport)}\\[2pt]\epsilon_k&\text{(direct)}\end{cases},
\qquad \mathbf{z}=\{z_k\}_{k=1}^{K}.
\label{eq:code}
\end{equation}
A transported coordinate passes the encoder's posterior through~\Cref{eq:transport}, while a direct
coordinate injects the discrete cluster identity of the routed prior component as illustrated in \Cref{fig:method}. The behaviour is
governed by three choices, the routing weights $\boldsymbol{\pi}$ in~\Cref{eq:route}, the
source reference coupling matched ($c_k\!\equiv\!k$, so $\epsilon_k$ and the transport
in~\Cref{eq:transport} share a component and $z_k=\mu^q_k+\sigma^q_k\eta_k$, $\eta_k\sim\Ncal(0,1)$),
or routed ($c_k\!\sim\!\Cat(\pi^p)$, so the selected component shifts and rescales the sample,
\Cref{prop:routed}), and the per-coordinate mode in~\Cref{eq:code}. We use the plain
code $\{z_k\}$; a responsibility-scaled code $\{\pi^q_k z_k\}$ is left to future work.

Two configurations anchor our study. The per-coordinate configuration transports every
coordinate with routed noise ($c_k\!\sim\!\Cat(\pi^p)$, mismatched), so all $K$ dimensions are routed
transports. The interleaved configuration places matched transports on half of the coordinates
and direct emissions (routed by $\pi^q$, tied across those coordinates) on the other half, isolating a
transported sub-space and an anchored prior-draw sub-space within the same $K$-dimensional code~\cite{dupont2018learningdisentangledjointcontinuous}.

\begin{figure}[!htb]
\centering
\resizebox{\textwidth}{!}{%
\begin{tikzpicture}[font=\small, node distance=7mm and 10mm,
  io/.style={draw, rounded corners, minimum height=0.7cm, fill=yellow!30},
  enc/.style={trapezium, draw, trapezium left angle=70, trapezium right angle=110,
              minimum height=0.95cm, fill=orange!22, align=center},
  dec/.style={trapezium, draw, trapezium left angle=110, trapezium right angle=70,
              minimum height=0.95cm, fill=teal!18, align=center},
  pp/.style={draw, rounded corners, fill=blue!8, align=center, inner sep=3pt},
  cc/.style={draw, rounded corners, fill=green!12, align=center, inner sep=3pt},
  tt/.style={draw, rounded corners, fill=red!8, align=center, inner sep=3pt},
  ar/.style={-{Stealth[length=2mm]}, thick}]

\node[io] (x) {$x$};
\node[enc, right=of x] (enc) {VAE Encoder\\$q_\phi(z|x)$};
\node[pp, right=of enc] (post) {$\ell_k,\ \mu^q_k,\ \sigma^q_k$};

\node[pp, below=1.5cm of enc] (prior) {Frozen GMM prior\\$\pi^p,\,M^p,\,S^p$};
\node[cc, right=of prior] (route) {route $c_k\!\sim\!\Cat(\pi^p)$,\\draw $\epsilon_k$ \,(\Cref{eq:route})};
\node[tt, right=of route] (transp) {transport \,(\Cref{eq:transport})\\$z_k=\mu^q_k+\tfrac{\sigma^q_k}{\sigma^p_k}(\epsilon_k-\mu^p_k)$};
\node[dec, right=of transp] (dec) {VAE Decoder\\$p_\theta(x|z)$};
\node[io, right=of dec] (xh) {$\hat{x}$};

\draw[ar] (x)--(enc);
\draw[ar] (enc)--(post);
\draw[ar] (prior)--(route);
\draw[ar] (route)--(transp);
\draw[ar] (transp)--(dec);
\draw[ar] (dec)--(xh);
\draw[ar] (post.south) to[out=-90,in=90]
   node[pos=0.45,right=1pt]{\scriptsize $\mu^q_k,\sigma^q_k$} (transp.north);
\end{tikzpicture}}%
\caption{Routed-transport sampling (one latent coordinate). The encoder produces per-coordinate
posterior parameters $(\ell_k,\mu^q_k,\sigma^q_k)$. A component is routed from the frozen mixture,
$c_k\sim\Cat(\pi^p)$, and its noise $\epsilon_k$ is transported~\Cref{eq:transport} onto the
posterior Gaussian to give $z_k$, which is decoded. The responsibilities $\pi^q$~\Cref{eq:piq} and the
prior enter the mixture KL~\Cref{eq:kl} of the training objective~\Cref{eq:elbo}.}
\label{fig:method}
\end{figure}
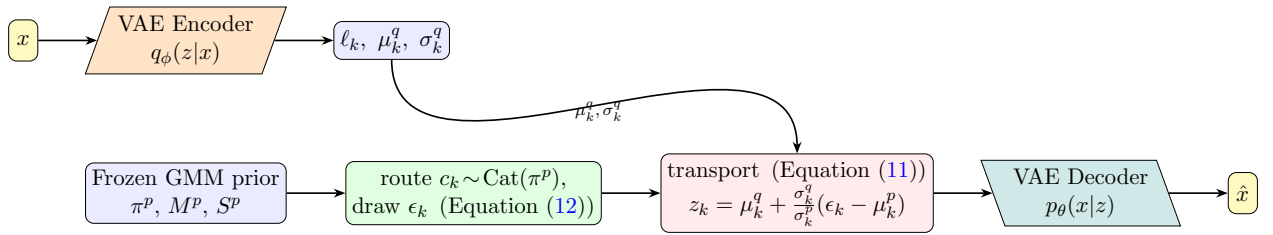

For the decoder \(p_\theta(x|z)\), we still apply the same regulations such as maximize the usual Evidence Lower Bound (ELBO).
Since we replace the conventional isotropic Gaussian prior in a VAE with a fixed, data-driven latent prior estimated from a pretrained autoencoder, the standard ELBO enforces per-sample alignment between the approximate posterior $q_\phi(z|x)$ and the prior, it does not guarantee that the aggregated posterior
\begin{equation}
    q_\phi(z) = \int q_\phi(z|x)\,p_{\text{data}}(x)\,dx
\end{equation}
matches the AE-derived latent distribution.
\begin{equation}
\mathcal{L}_{\mathrm{X-VAE}}(x)
=
\mathbb{E}_{q_\phi(z | x)}[\log p_\theta(x | z)]
-
D_{\mathrm{KL}}\big(q_\phi(z|x) \,\Vert\, p_{\text{AE}}(z)\big)
\end{equation}
For the first approach, the prior distribution from Autoencoder \(p_{\mathrm{AE}}(z)\) replaces the usual \(\mathcal{N}(0, I)\).
Thus, the KL divergence has a closed‑form:  
\begin{align}  
D_{\mathrm{KL}}(q_\phi(z|x) \,\Vert\, p_{\mathrm{AE}}(z)) &= \frac{1}{2} \sum_{j=1}^{d} \Biggl[ \log \frac{\sigma_{\mathrm{AE},j}^2}{\sigma_{\phi,j}(x)^2}  
+ \frac{\sigma_{\phi,j}(x)^2+(\mu_{\phi, j} - \mu_{AE, j})}{\sigma_{\mathrm{AE},j}^2}  
- 1 \Biggr]
\end{align}  
where \(d\) is the latent dimensionality, \(\mu_{\phi,j}(x)\) / \(\sigma_{\phi,j}(x)\) the \(j\)-th dimension of the output of encoder, and \(\mu_{\mathrm{AE},j}\), \(\sigma_{\mathrm{AE},j}\) are the prior’s per‑dimension statistics (see \Cref{kl_derive}).
The reparameterization trick~\cite{kingma2022vae, rezende2014} is still employed to draw latent samples from \(z = \mu_\phi(x) + \sigma_\phi(x) \odot \epsilon\), \(\epsilon \sim \mathcal{N}(0,I)\), enabling back‑propagation through the stochastic sampling.
We then optimize \(\phi, \theta\) to maximize the empirical expectation of \(\mathcal{L}(x)\) over the dataset.
The  overall flowchart structure is illustrated in \Cref{fig:method_flowchart}

\begin{figure}[!htb]
    \centering
    \includegraphics[width=0.75\linewidth]{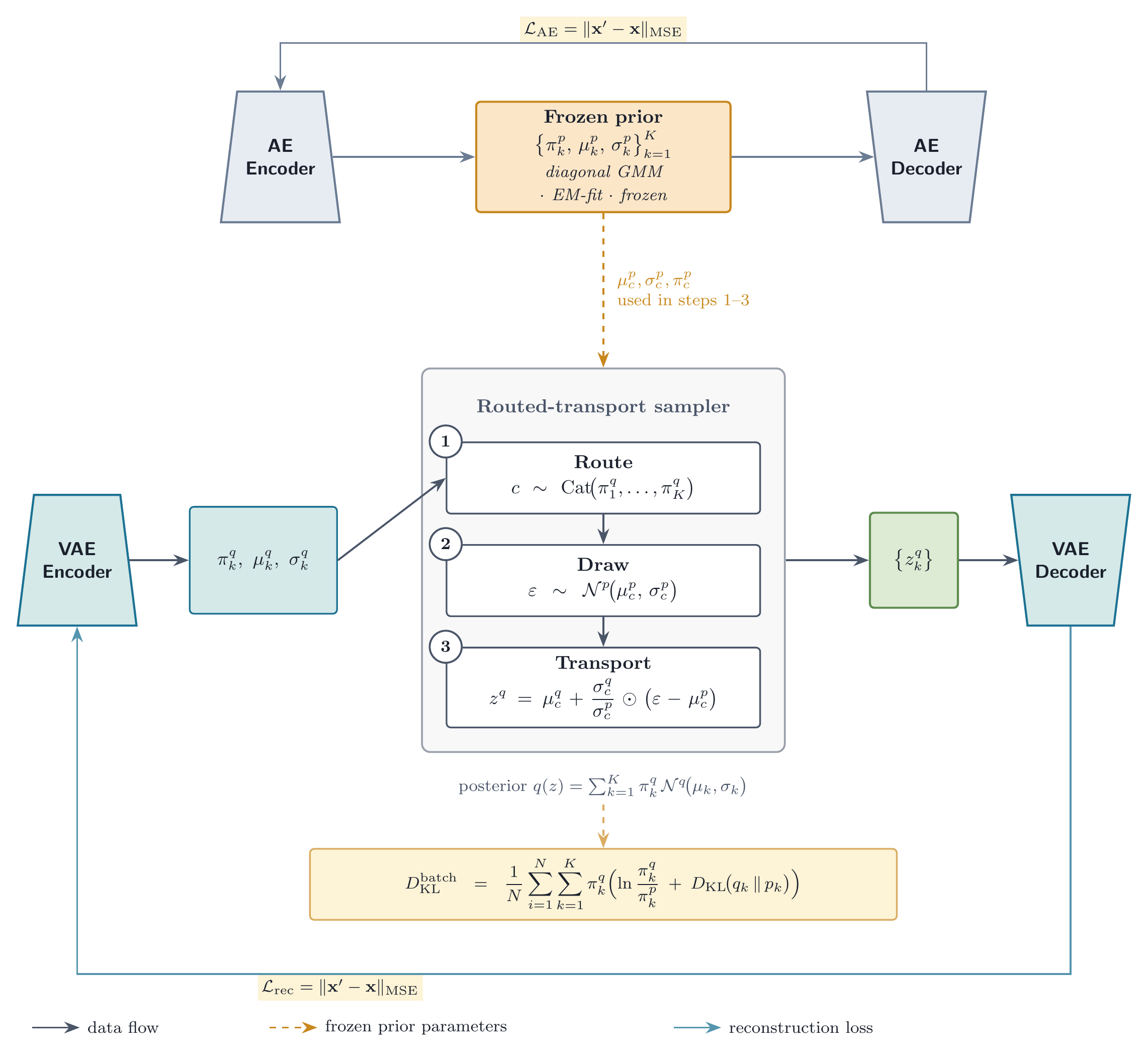}
    \caption{Architecture of the proposed X-VAE. \emph{Top:} a deterministic
autoencoder is trained first, and a $K$-component diagonal Gaussian mixture is
fit once on its latent codes to give the fixed prior
$\{\pi^{p}_{k},\mu^{p}_{k},\sigma^{p}_{k}\}_{k=1}^{K}$~\eqref{eq:gmm-prior},
which is frozen during VAE training. \emph{Bottom:} the VAE encoder emits a
per-coordinate Gaussian posterior and routing weights
$(\pi^{q}_{k},\mu^{q}_{k},\sigma^{q}_{k})$; each latent coordinate is then formed
by the routed-transport sampler---route a component $c\sim\mathrm{Cat}(\pi^{q})$,
draw $\varepsilon\sim\mathcal{N}^{p}(\mu^{p}_{c},\sigma^{p}_{c})$, and transport
it to $z^{q}=\mu^{q}_{c}+(\sigma^{q}_{c}/\sigma^{p}_{c})\odot(\varepsilon-\mu^{p}_{c})$~\eqref{eq:code}.
Training combines the reconstruction MSE with the closed-form two-part mixture
KL~\eqref{eq:kl} (a categorical term plus a per-component Gaussian term). Dashed
arrows mark the frozen prior parameters; solid arrows mark data flow.}
    \label{fig:method_flowchart}
\end{figure}

For the second approach, we maximize the usual ELBO with the fixed mixture prior~\Cref{eq:gmm-prior} in place of $\Ncal(0,I)$,
\begin{equation}
    \mathcal{L}(x)=\E_{q_\phi(z|x)}\big[\log p_\theta(x|z)\big]-\KL\!\big(q_\phi(z|x)\,\big\|\,p(z)\big),
\label{eq:elbo}
\end{equation}
and bound the mixture KL by its sampled upper bound (\Cref{lem:mixkl}),
\begin{equation}
    \KL\big(q_\phi(z|x)\,\big\|\,p(z)\big)\;\le\;
    \KL(\pi^q\,\|\,\pi^p)+\sum_{k=1}^{K}\pi^q_k\,\KL\!\big(\Ncal^q_k\,\big\|\,\Ncal^p_k\big),
\label{eq:kl}
\end{equation}
where $\Ncal^p_k=\Ncal(\mu^p_k,\sigma^{p2}_k)$ and each Gaussian term is closed form (~\Cref{lem:gausskl}),
\begin{equation}
    \KL\!\big(\Ncal^q_k\,\big\|\,\Ncal^p_k\big)
    =\log\frac{\sigma^p_k}{\sigma^q_k}+\frac{\sigma^{q2}_k+(\mu^q_k-\mu^p_k)^2}{2\,\sigma^{p2}_k}-\frac12.
\label{eq:gausskl}
\end{equation}
The decoder term in~\Cref{eq:elbo} is the usual Gaussian/Bernoulli reconstruction loss, and the
transport~\Cref{eq:transport} preserves the reparameterization gradient so that
$\phi,\theta$ are trained end to end.

At generation time the encoder is not used, each coordinate is drawn from its routed prior marginal,
\begin{equation}
    c_k\sim\Cat(\pi^p),\qquad z_k\sim\Ncal\!\big(M^p_{c_k,k},\,S^{p2}_{c_k,k}\big),
\label{eq:gen}
\end{equation}
and $\mathbf{z}=\{z_k\}$ is decoded. Because the prior is anchored to the AE manifold, samples remain
faithful to the data's cluster structure.
Optionally, during generation we also introduce a scalar \(s > 0\) to modulate the variance, i.e., use
\(\mathcal{N}( \mu_{\mathrm{AE}}, \, \mathrm{diag}((s \, \Sigma_{\mathrm{AE}})^2) )\), to control sample diversity and quality.
After training, when generating novel samples, we sample from  
\begin{equation}
\tilde p(z) = \mathcal{N}\bigl(z; \mu_{\mathrm{AE}}, \,\mathrm{diag}((\alpha\, \sigma_{\mathrm{AE}})^2)\bigr),
\end{equation}
where \(\alpha\) is a hyper‑parameter (e.g., \(\alpha \ge 1\)) that controls generation variance. This allows to tune between sample diversity (larger \(\alpha\)) and fidelity (smaller \(\alpha\)).
A small $\alpha$ near $0$ will lead a small standard deviation that cause the generation close to the average while a large on could lead a more diverse generation. 
By scaling the latent standard deviation or shifting the mean allows controlled exploration, increase the size of sample space it can draw from to improve the generative diversity.

\section{Experiments}
\label{sec:experiments}

We mainly evaluate X-VAE on MNIST ($28\times28$)~\cite{lecun1989}, CelebA ($64\times64$)~\cite{liu2015faceattributes}, and the synthetic clustered~\cite{yacoby2022failuremodesvariationalautoencoders} benchmark is discussed separately below.
The Synthetic clusters has fully known structure and distributions, $K=3$ matched to the ground-truth modes.
A controlled diagnostic for whether the mixture prior recovers and the routed-transport posterior preserves known discrete structure, with no image-complexity confound.
MNIST is the intermediate case, real pixels but with a known, well-separated discrete class structure (ten digits).
It helps testing whether the controlled-condition behavior transfers to real images.
CelebA is the realistic, hardest regime, since there is no clean categories (continuous, overlapping attributes), so the number of modes and latent geometry are unknown and must be inferred from data.
This is where FID is decisive and where the configurations, routing or transport split, are compared to pick the best model.
Training proceeds in two stages.
At the first stage, a deterministic autoencoder is trained, and its latent codes are used to estimate the prior in one of two ways.
(i) the per-coordinate empirical mean and standard deviation, giving a single data-driven Gaussian $\Ncal(\mu_{AE},\sigma_{AE}^2)$ that replaces $\Ncal(0,I)$ (the single-Gaussian X-VAE) follows by the \Cref{equ:single_gaussian} or (ii) a $K$-component diagonal Gaussian mixture~\Cref{eq:gmm-prior} which, together with the dimension-wise routed-transport posterior, forms the full X-VAE.
A VAE with this fixed prior is then trained by maximizing the ELBO~\Cref{eq:elbo}.
The VAE and AE architectures need not coincide, but their latent dimensionalities must match so that the latent statistics transfer, which are illustrated in detailed architectures and settings in \Cref{imp_detail}.

We compare against a standard VAE~\cite{kingma2022vae}, a VampPrior VAE~\cite{tomczak2018vamp}, a hierarchical VAE~\cite{vahdat2021Hierarchy,sonderby2016ladder}, a standard mixture-of-Gaussians prior VAE, and a Gumbel--Softmax GM-VAE~\cite{jang2017gumbel,dilokthanakul2017}.
All models share the latent budget $K$.
Our method exposes two families of discrete choices, the routing of the anchored coordinate, which are weights $\boldsymbol{\pi}\!\in\!\{\pi^p,\pi^q\}$ and Gaussian source from either $\Ncal^p$ or $\Ncal^q$.
To maintain the same latent dimension of all models, our methods will also keep to total latent dimension $K$.
Among the $K$ dimensions, $n$ out of $K$ will be transport split and the rest will be drawn from the prior.
For instance, the $(K{-}n) + n$ variants, and the transport split, the fraction of the $(K{-}n)$ latent coordinates that carry the encoder transport, the rest $n$ drawn from the prior as illustrated in \Cref{eq:code} and \Cref{fig:method}.
The splits will include $K {-} 1 + 1$, $\frac{K}{3}$, $\frac{2K}{3}$, and $K$ dimensional transport without any drawn for the second approach.
We report reconstruction MSE, Fr\'echet Inception Distance (FID)~\cite{heusel2017fid} computed with Inception-V3~\cite{szegedy2015inception}, the Inception Score (IS)~\cite{salimans2016is}, and the train/test objective with its reconstruction and KL terms.
As the full sweep contains many configurations, the tables below report the baselines together with three of our strongest configurations with the complete per-model results appear in \Cref{app:fullresults}.

\Cref{tab:loss} illustrated the final-epoch train/test objective with its reconstruction and KL terms.
For the routed-transport models the KL is the sum of two terms~\Cref{eq:kl}.
A categorical term $\KL(\pi^q\,\|\,\pi^p)$ that matches the batch component-assignment weights to the prior weights, and a per-component Gaussian term $\sum_k \pi^q_k\,\KL(\Ncal^q_k\,\|\,\Ncal^p_k)$, the value reported in the table is their sum.
On CelebA our configurations attain the lowest test objective of any model ($164.65$, below the standard VAE's $165.19$ and VampPrior's $166.22$) while keeping a moderate KL ($\approx56$); the hierarchical VAE reaches the smallest KL ($48.58$) but pays for it with a much larger reconstruction term ($124.64$) and worse FID.
On MNIST the objectives are tightly clustered ($28.5$--$28.8$ for our models, hierarchical VAE lowest at $28.10$).
The GM-VAE is a clear outlier on both datasets ($69.17$ test on MNIST, $232.07$ on CelebA).
Its enormous KL ($59.33$, $128.21$) explains its poor FID despite the sharpest reconstructions, as illustrated in \Cref{fig:select_recon_mnist_celeba} with complete reconstructions of images in \Cref{app:full_recon}.
As illustrated in \Cref{fig:loss-mnist}, \Cref{fig:loss-celeba}, the losses for all models are successfully converge at the end epoch.
\begin{table}[!htb]
\centering\footnotesize
\setlength{\tabcolsep}{4pt}
\caption{Final-epoch train/test objective with reconstruction and KL terms on MNIST and CelebA. Lowest
\emph{test} objective per dataset in bold; magnitudes are not comparable across datasets. For our
routed-transport models the KL is the sum of a categorical and a per-component Gaussian term
(\Cref{eq:kl}).}
\label{tab:loss}
\begin{tabular}{@{}l cccc cccc@{}}
\toprule
 & \multicolumn{4}{c}{MNIST ($K{=}64$)} & \multicolumn{4}{c}{CelebA ($K{=}256$)} \\
\cmidrule(lr){2-5}\cmidrule(lr){6-9}
Model & train & test & recon & KL & train & test & recon & KL \\
\midrule
Standard VAE       & 28.96 & 28.85 & 16.67 & 12.29 & 167.34 & 165.19 & 111.75 & 55.59 \\
VampPrior VAE      & 28.59 & 28.30 & 17.51 & 11.08 & 166.09 & 166.22 & 111.51 & 54.57 \\
Hierarchical VAE   & 28.31 & $\mathbf{28.10}$ & 16.64 & 11.67 & 173.22 & 170.79 & 124.64 & 48.58 \\
Standard MoG VAE   & 29.50 & 29.24 & 17.70 & 11.80 & 167.87 & 167.96 & 112.89 & 54.98 \\
GM-VAE             & 69.37 & 69.17 & 10.03 & 59.33 & 231.67 & 232.07 & 103.47 & 128.21 \\
\midrule
Ours, $(K{-}1)+1$ ($\pi^q,\Ncal^p$)   & 28.82 & 28.61 & 16.70 & 12.13 & 166.64 & 164.82 & 110.61 & 56.03 \\
Ours, transport $\tfrac{2}{3}K$   & 29.08 & 28.75 & 16.98 & 12.10 & 166.59 & 165.42 & 110.58 & 56.01 \\
Ours, $(K{-}1)+1$ ($\pi^q,\Ncal^q$)   & 28.91 & 28.56 & 16.88 & 12.04 & 167.88 & $\mathbf{164.65}$ & 111.95 & 55.93 \\
\bottomrule
\end{tabular}
\end{table}

\begin{figure}[!htb]
    \centering
    \includegraphics[width=0.47\linewidth]{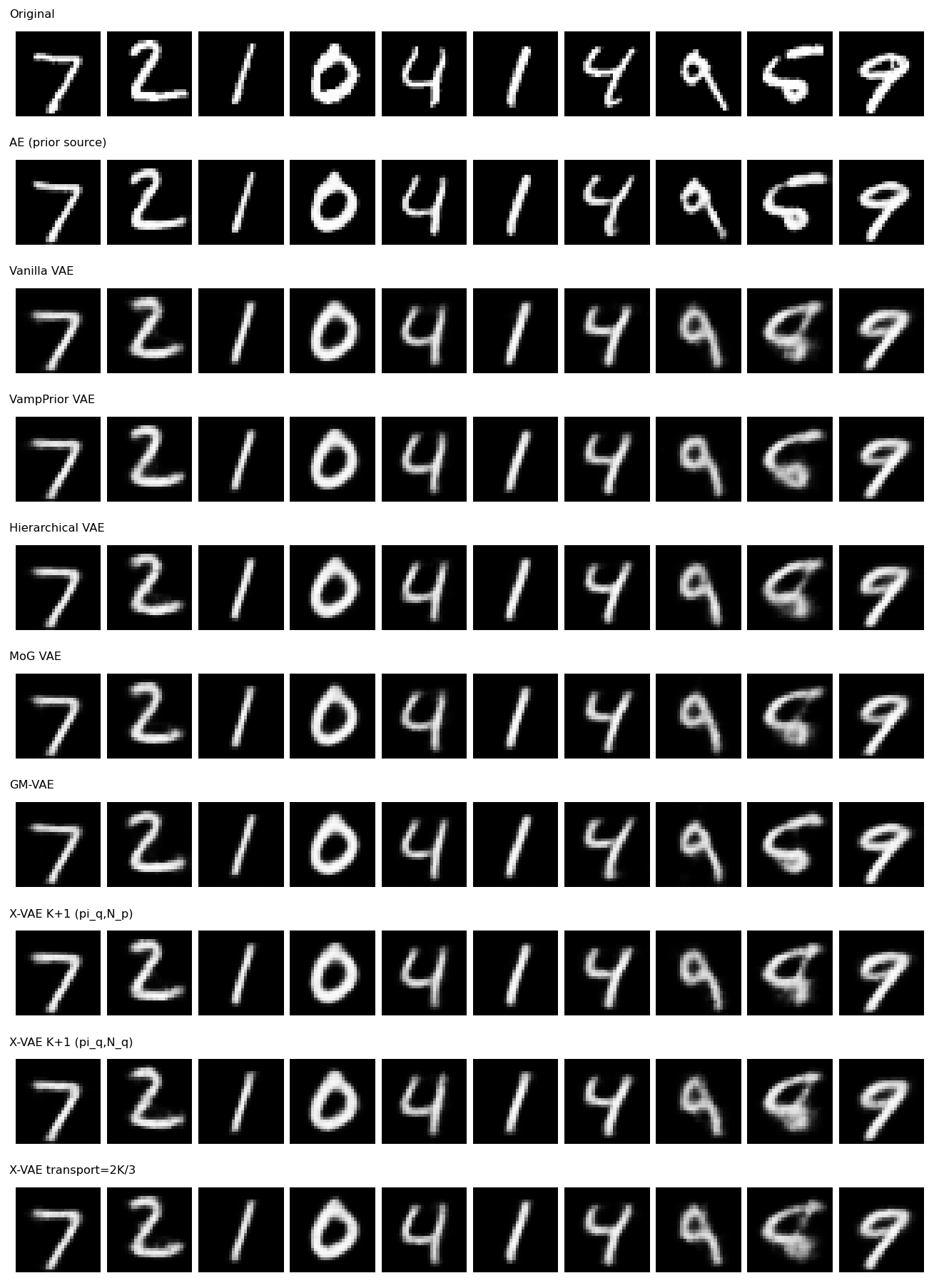}
    \includegraphics[width=0.47\linewidth]{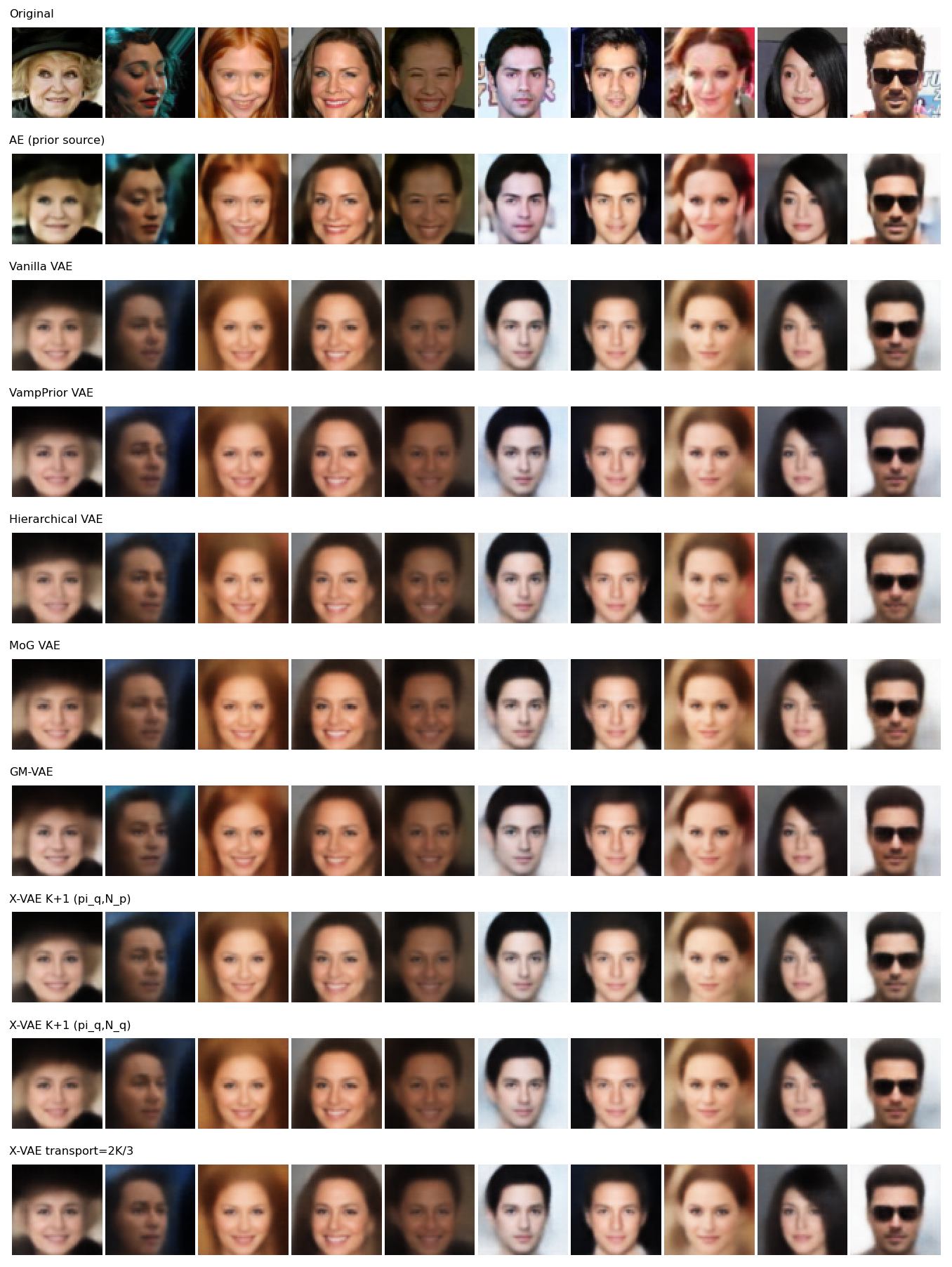}
    \caption{Left: MNIST reconstructions, Right: Celeba reconstructions}
    \label{fig:select_recon_mnist_celeba}
\end{figure}

\begin{table}[!htb]
\centering\small
\setlength{\tabcolsep}{4.5pt}
\caption{Reconstruction error (MSE), FID, and Inception Score (IS) on MNIST and CelebA. Lower is better
for MSE and FID, higher for IS; best per column in bold. We list the baselines and three of our strongest configurations; the complete per-model sweep is in \Cref{app:fullresults}.}
\label{tab:quality}
\begin{tabular}{@{}l ccc ccc@{}}
\toprule
 & \multicolumn{3}{c}{MNIST ($K{=}64$)} & \multicolumn{3}{c}{CelebA ($K{=}256$)} \\
\cmidrule(lr){2-4}\cmidrule(lr){5-7}
Model & Rec$\downarrow$ & FID$\downarrow$ & IS$\uparrow$ & Rec$\downarrow$ & FID$\downarrow$ & IS$\uparrow$ \\
\midrule
Standard VAE       & 11.24 & 49.67 & $2.382{\pm}.068$ & 87.42 & 86.34 & $1.816{\pm}.052$ \\
VampPrior VAE      & 12.12 & 47.76 & $2.460{\pm}.093$ & 90.27 & 82.80 & $\mathbf{1.858}{\pm}.047$ \\
Hierarchical VAE   & 11.89 & \textbf{46.64} & $2.411{\pm}.065$ & 104.60 & 87.09 & $1.806{\pm}.073$ \\
Standard MoG VAE   & 12.88 & 58.27 & $\mathbf{2.631}{\pm}.101$ & 92.63 & 87.49 & $1.795{\pm}.034$ \\
GM-VAE             & \textbf{4.98} & 79.08 & $2.062{\pm}.063$ & \textbf{74.99} & 89.14 & $1.781{\pm}.037$ \\
\midrule
Ours, $K{-}1+1$ ($\pi^q,\Ncal^p$)   & 11.89 & 47.41 & $2.402{\pm}.071$ & 87.99 & $\mathbf{80.40}$ & $1.772{\pm}.060$ \\
Ours, transport $\tfrac{2}{3}K$   & 11.65 & 47.54 & $2.382{\pm}.106$ & 88.24 & 81.90 & $1.807{\pm}.070$ \\
Ours, $K{-}1+1$ ($\pi^q,\Ncal^q$)   & 11.60 & 47.21 & $2.411{\pm}.079$ & 88.40 & 83.11 & $1.797{\pm}.032$ \\
\bottomrule
\end{tabular}
\end{table}

\Cref{tab:quality} reports reconstruction MSE, FID, and IS on MNIST and CelebA.
On CelebA, our $K{-}1+1$ $(\pi^q,\Ncal^p)$ configuration attains the best FID of any model ($80.40$), ahead of the strongest baseline (VampPrior, $82.80$) and well ahead of the standard VAE ($86.34$).
A second configuration (transport $\tfrac{2}{3}K$, FID $81.90$) also surpasses every baseline.
On MNIST, our configurations improve clearly over the standard VAE (FID $47.2$--$47.5$ vs.\ $49.67$) and over VampPrior ($47.76$), and the single best configuration in the full sweep, transport $\tfrac{1}{3}K$, FID $46.01$ (\Cref{app:fullresults}), edges the strongest baseline (hierarchical VAE, $46.64$).
The Gumbel--Softmax GM-VAE shows a pronounced reconstruction and generation trade-off.
By far the lowest reconstruction error on both datasets ($4.98$ on MNIST, $74.99$ on CelebA) but the worst FID ($79.08$, $89.14$), consistent with its very large KL term ( \ref{tab:loss}). Across our configurations the metrics move only modestly, indicating robustness to the routing distribution and the transport split. \Cref{fig:recon-mnist},\Cref{fig:gen-mnist} and\Cref{fig:recon-celeba},\Cref{fig:gen-celeba} show qualitative results.
\begin{figure}[!htb]
\centering
\includegraphics[width=\linewidth]{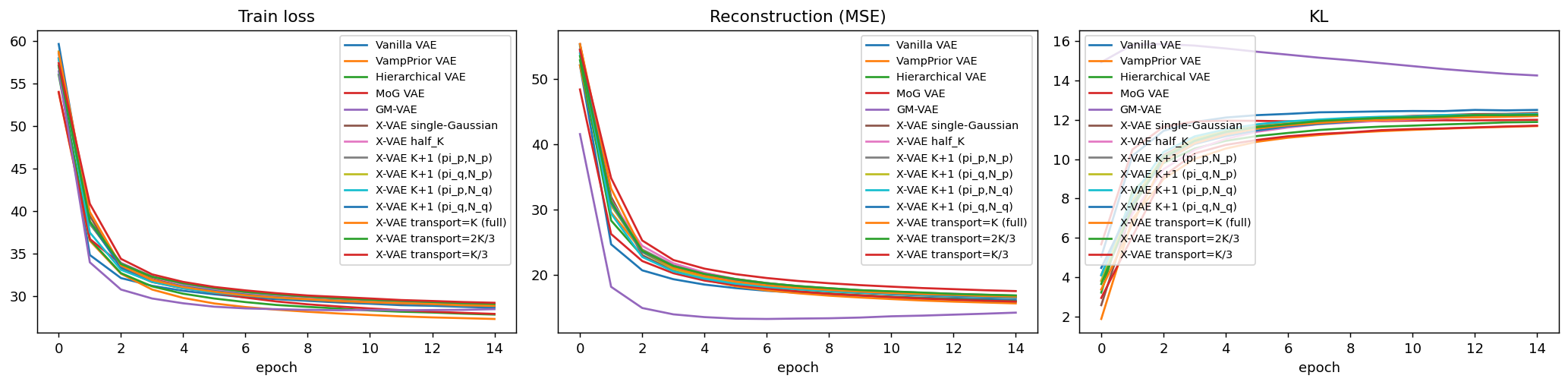}
\caption{MNIST training curves (total / reconstruction / KL vs.\ epoch) for our method and the
baselines.}
\label{fig:loss-mnist}
\end{figure}
\begin{figure}[!htb]
\centering
\includegraphics[width=\linewidth]{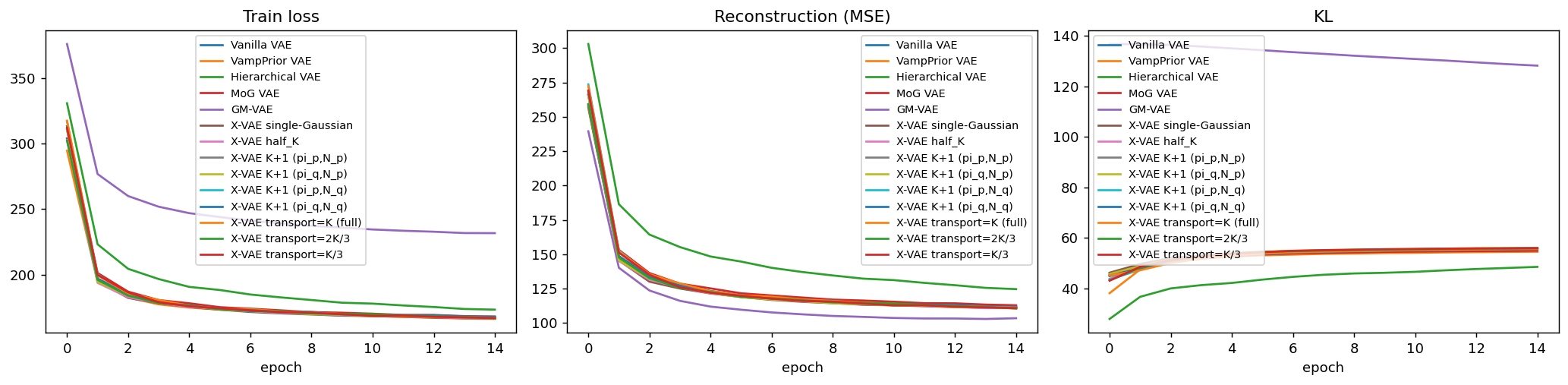}
\caption{CelebA training curves (total / reconstruction / KL vs.\ epoch).}
\label{fig:loss-celeba}
\end{figure}

On the synthetic three-cluster benchmark ($K{=}3$), a standard VAE constrained by an isotropic prior fails to reproduce the multimodal structure, whereas the AE-anchored prior preserves the three modes in both reconstruction and generation (\Cref{fig:recon-clusters} and\Cref{fig:gen-clusters}). \Cref{tab:loss-clusters} reports the final-epoch objective with its reconstruction and KL terms.
Our well-behaved configurations reconstruct as accurately as the baselines (reconstruction $\approx0.011$) but carry a slightly larger KL ($\approx0.040$--$0.042$ vs.\ $\approx0.036$--$0.038$) and hence a marginally higher objective---the price of pulling the aggregated posterior toward the cluster means rather than toward an isotropic prior.
The Gumbel--Softmax GM-VAE reaches the smallest objective ($0.007$) but with a degenerate, near-zero/negative sampled KL: its discrete code carries little information on this low-dimensional data, mirroring the prior--posterior collapse seen on MNIST and CelebA.

A second, configuration-level effect is visible only here. Because $K{=}3$ is small, the configurations that transport a single coordinate---the interleaved model and transport $\tfrac{1}{3}K$, both of which reduce to one transport dimension at $K{=}3$, reconstruct markedly worse (reconstruction $\approx0.09$, objective $\approx0.12$) than configurations that transport two or more coordinates ($\approx0.011$--$0.015$; see \Cref{app:fullresults}).
This directly illustrates that the transported coordinates carry the reconstruction: too small a transport fraction sacrifices fidelity, consistent with the dataset-dependent behavior discussed in \Cref{sec:discussion}.

\begin{table}[!htb]
\centering\footnotesize
\setlength{\tabcolsep}{5pt}
\caption{Clustered data ($K{=}3$): final-epoch train/test objective with reconstruction and KL terms for
the baselines and our three featured configurations. GM-VAE attains the smallest objective but with a
degenerate (near-zero/negative) sampled KL; the full per-model sweep is in \Cref{app:fullresults}.}
\label{tab:loss-clusters}
\begin{tabular}{@{}l cccc@{}}
\toprule
Model & train & test & recon & KL \\
\midrule
Standard VAE       & 0.048 & 0.048 & 0.010 & 0.038 \\
VampPrior VAE      & 0.049 & 0.051 & 0.011 & 0.038 \\
Hierarchical VAE   & 0.047 & 0.047 & 0.011 & 0.036 \\
Standard MoG VAE   & 0.047 & 0.046 & 0.011 & 0.036 \\
GM-VAE             & 0.008 & 0.007 & 0.010 & $-0.002$ \\
\midrule
Ours, $K{-}1+1$ ($\pi^q,\Ncal^p$)   & 0.053 & 0.051 & 0.011 & 0.042 \\
Ours, transport $\tfrac{2}{3}K$   & 0.055 & 0.051 & 0.013 & 0.041 \\
Ours, $K{-}1+1$ ($\pi^q,\Ncal^q$)   & 0.056 & 0.057 & 0.015 & 0.040 \\
\bottomrule
\end{tabular}
\end{table}

\begin{figure}[!htb]
\centering
\includegraphics[width=0.6\linewidth]{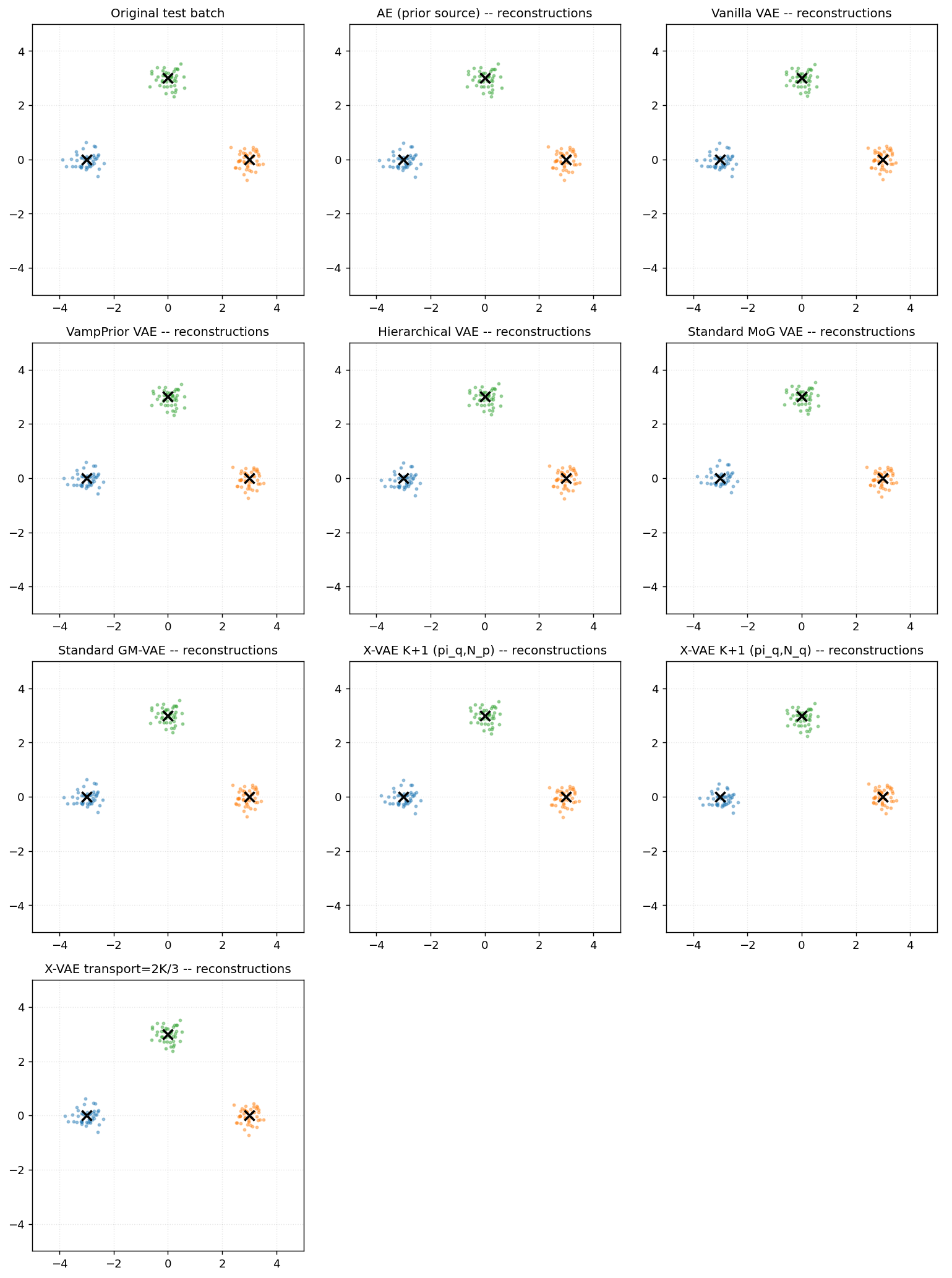}
\caption{Clustered data: original two-dimensional data, reconstructions by our method, and
reconstructions by a standard VAE.}
\label{fig:recon-clusters}
\end{figure}

\begin{figure}[!htb]
\centering
\includegraphics[width=\linewidth]{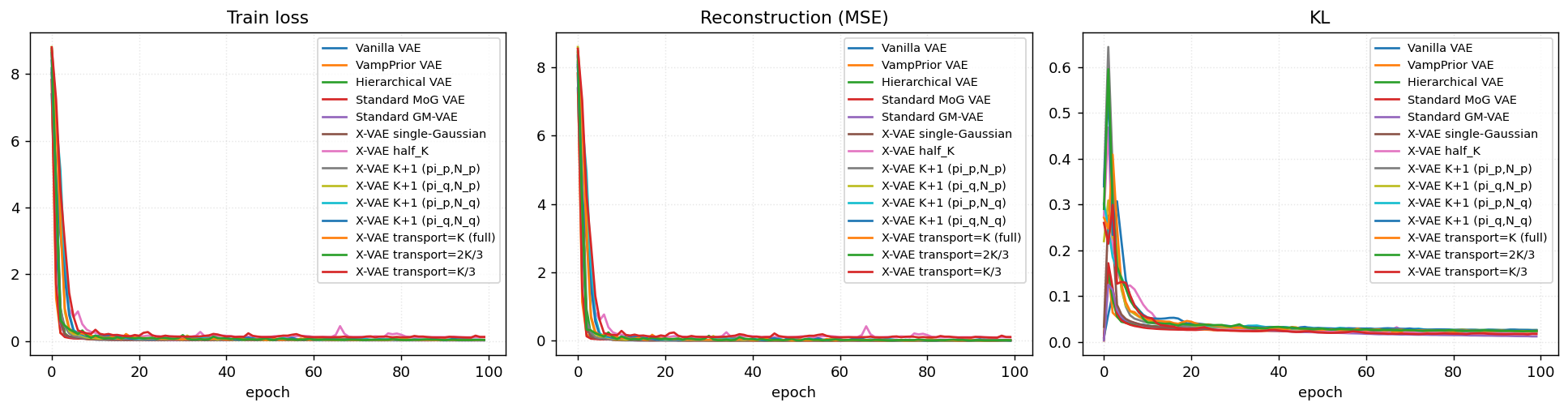}
\caption{Clustered-data training curves (total / reconstruction / KL vs.\ epoch).}
\label{fig:loss-clusters}
\end{figure}

\section{Discussion}
\label{sec:discussion}

The experiments support a clear, and in places strengthened, picture.
Grounding the VAE prior in statistics of the AE latent space, either a single data-driven Gaussian or an AE-fit Gaussian mixture with the routed-transport posterior.
It yields a model that is competitive with, and on the harder benchmark simple and better than, standard baselines, without adding any trainable prior parameters.
On CelebA our best configuration attains the lowest FID of all models ($80.40$, vs.\ $82.80$ for the strongest baseline), and our configurations reach the lowest test objective overall ($164.65$, ~\Cref{tab:loss}), while on MNIST the best configuration in the full sweep also edges the strongest baseline on FID ($46.01$ vs.\ $46.64$), while every configuration improves clearly over the standard VAE.
Because the prior is obtained once from latent statistics with a closed-form KL~\Cref{eq:kl}, these gains come at essentially no additional training cost.

The sample grids mirror the metric story
(\Cref{fig:select_mnist_gen,fig:select_celeba_gen}; the complete per-model grids are in
\Cref{fig:gen-mnist,fig:gen-celeba}). On CelebA the routed-transport configurations produce
coherent, well-formed faces with varied pose, lighting, gender, and expression; their facial geometry
and skin and hair tones are as sharp as and in several rows visibly sharper than the flexible-prior
baselines, consistent with the best-of-all FID ($80.40$, \Cref{tab:quality}). The standard VAE is by
comparison smoother and more washed-out, and GM-VAE, despite the lowest reconstruction error, yields
the least face-like samples: the qualitative signature of its inflated KL and the prior--posterior
collapse already visible in \Cref{tab:loss}.
\begin{figure}[!htb]
    \centering
    \includegraphics[width=0.9\linewidth]{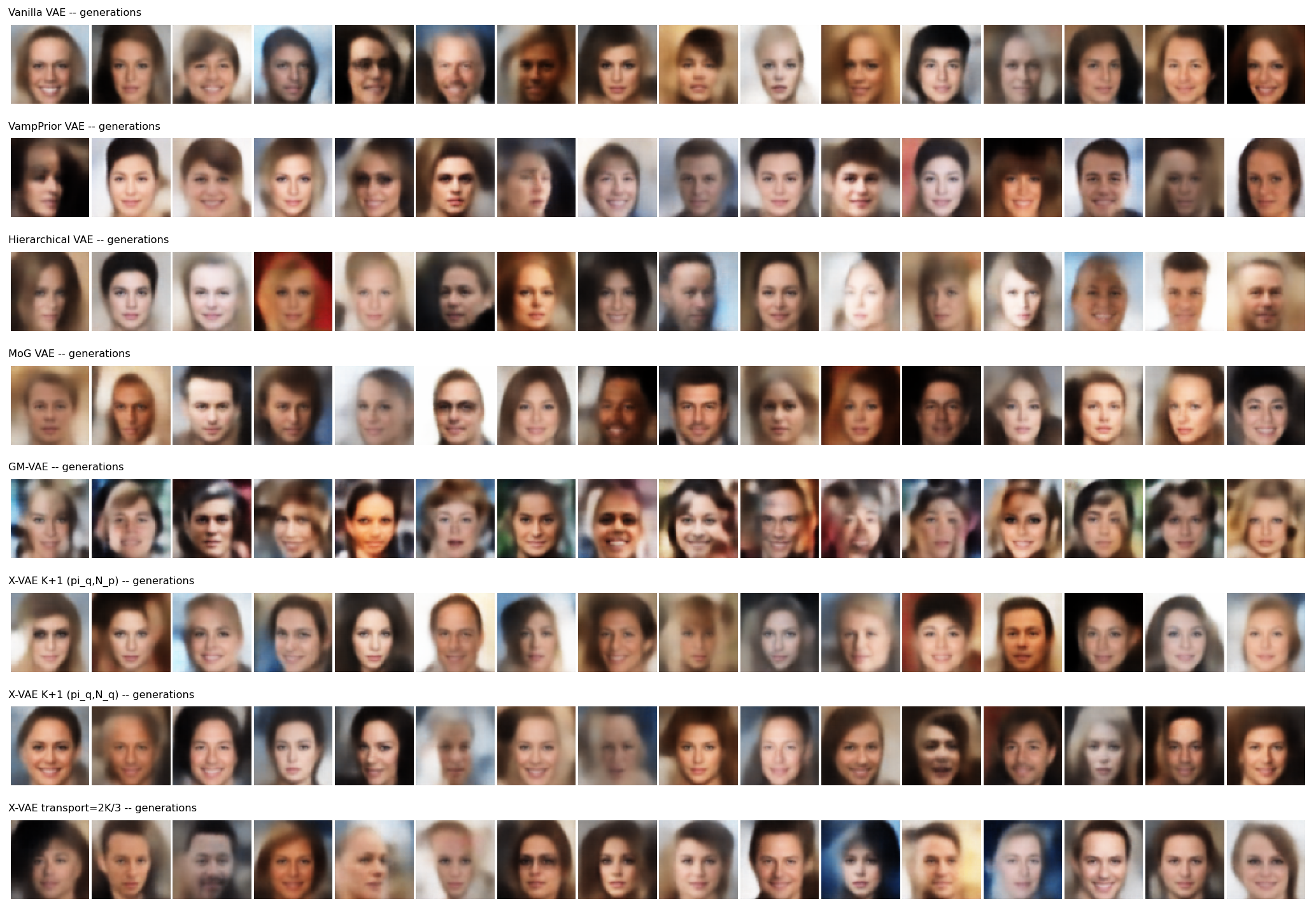}
    \caption{Generated samples for CelebA}
    \label{fig:select_celeba_gen}
\end{figure}

On MNIST the comparison is more even-handed, and the samples make the parity in
\Cref{tab:quality} concrete. Our generations are legible, cover all ten digit classes, and show good
stroke variety, clearly improving on the standard VAE and matching VampPrior and the MoG VAE; but they
do not dominate. The crispest baselines (hierarchical and MoG VAE) render marginally cleaner strokes,
matching their slightly better MNIST FID, and a few of our prior-drawn configurations---most visibly
the single-Gaussian and $(\pi^p,\Ncal^q)$ variants---occasionally emit broken or blobby digits in the
coordinates that are sampled from the prior rather than transported. The cleanest rows are the
$(\pi^q,\Ncal^q)$ and small-transport ($\tfrac{1}{3}K$) configurations, exactly where the best MNIST
FID ($46.01$) sits. We therefore read MNIST as a parity result rather than a win: the AE-anchored prior
delivers a clear gain over the standard VAE without the crispness edge that the parameter-heavy
flexible-prior baselines obtain on this low-resolution, low-diversity benchmark. We attribute the gap
to the limited room a $28\times28$ grayscale manifold leaves for a fixed prior to help---the
reconstruction term, not the prior, dominates the objective here (\Cref{tab:loss})---whereas the
richer CelebA 
\begin{figure}[!htb]
    \centering
    \includegraphics[width=0.9\linewidth]{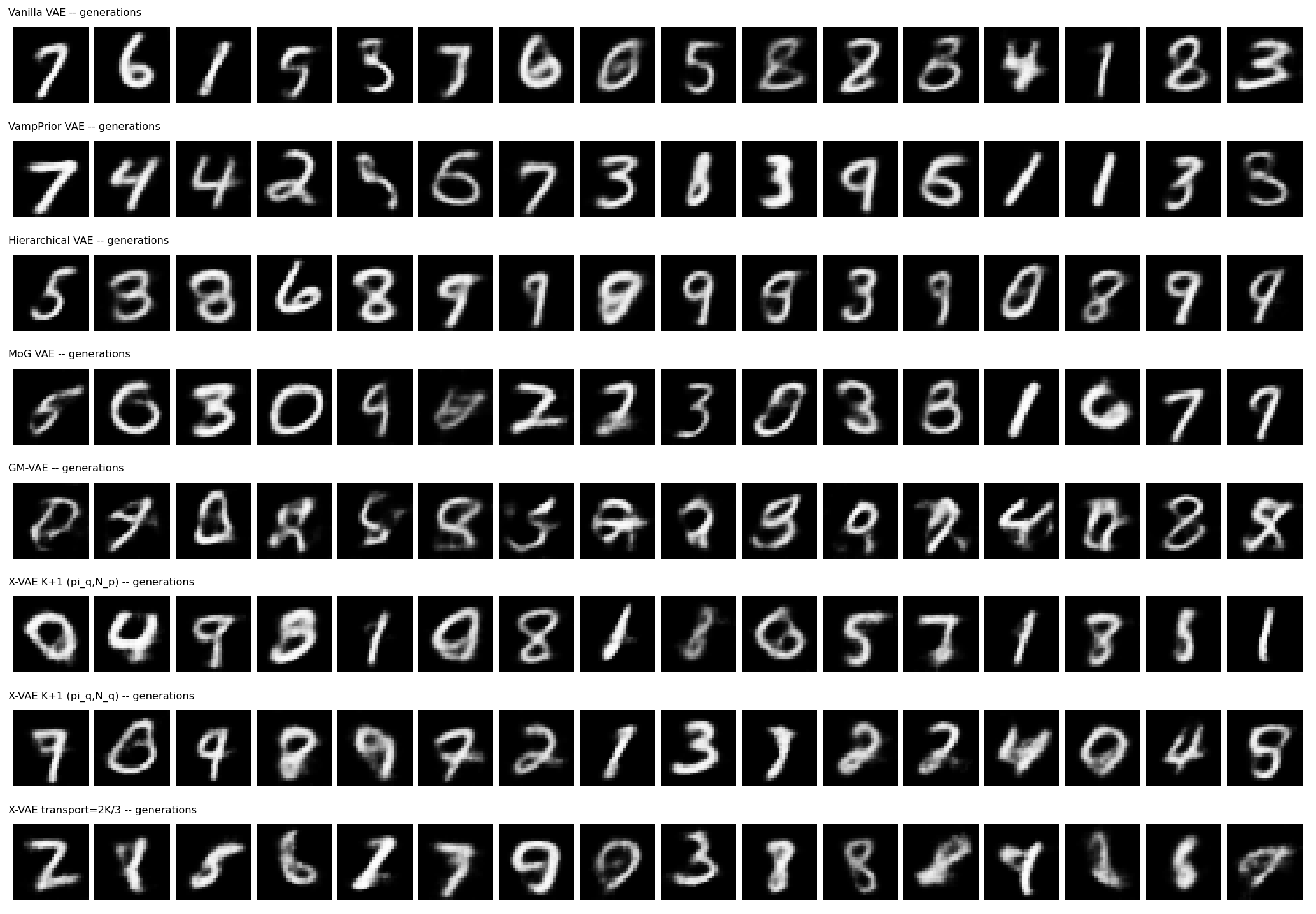}
    \caption{Generated samples for MNIST}
    \label{fig:select_mnist_gen}
\end{figure}

The clustered data exposes the routing-and-split mechanism most directly
(\Cref{fig:gen-clusters}). The well-behaved configurations generate samples that stay inside the three
Gaussian modes, whereas the standard VAE leaks probability mass into the inter-cluster gaps---the
multimodal failure the AE-anchored prior is designed to prevent. The configurations that transport only
a single coordinate (the half-$K$ and transport-$\tfrac{1}{3}K$ panels, both of which reduce to one
transport dimension at $K{=}3$) instead collapse their samples onto a one-dimensional curve threaded
through the cluster centres. This is the visual counterpart of their inflated reconstruction term in
\Cref{tab:loss-clusters}, it confirms that the transported sub-space, not the prior-drawn one, carries
the geometry, so too small a transport fraction sacrifices coverage of the data manifold.

On the synthetic clustered data the AE-anchored prior preserves the three modes that an isotropic prior collapses, at the cost of a slightly larger KL than the baselines, the expected signature of pulling the aggregated posterior toward the cluster means rather than the origin as illustrated in \Cref{fig:select_cluster_gen}. This benchmark also exposes the transport fraction most sharply, at $K{=}3$, the configurations that transport a single coordinate lose reconstruction fidelity (\Cref{tab:loss-clusters}. \Cref{app:fullresults}), confirming that the transported sub-space, not the prior-drawn one, carries the reconstruction.

\begin{figure}[!htb]
    \centering
    \includegraphics[width=0.8\linewidth]{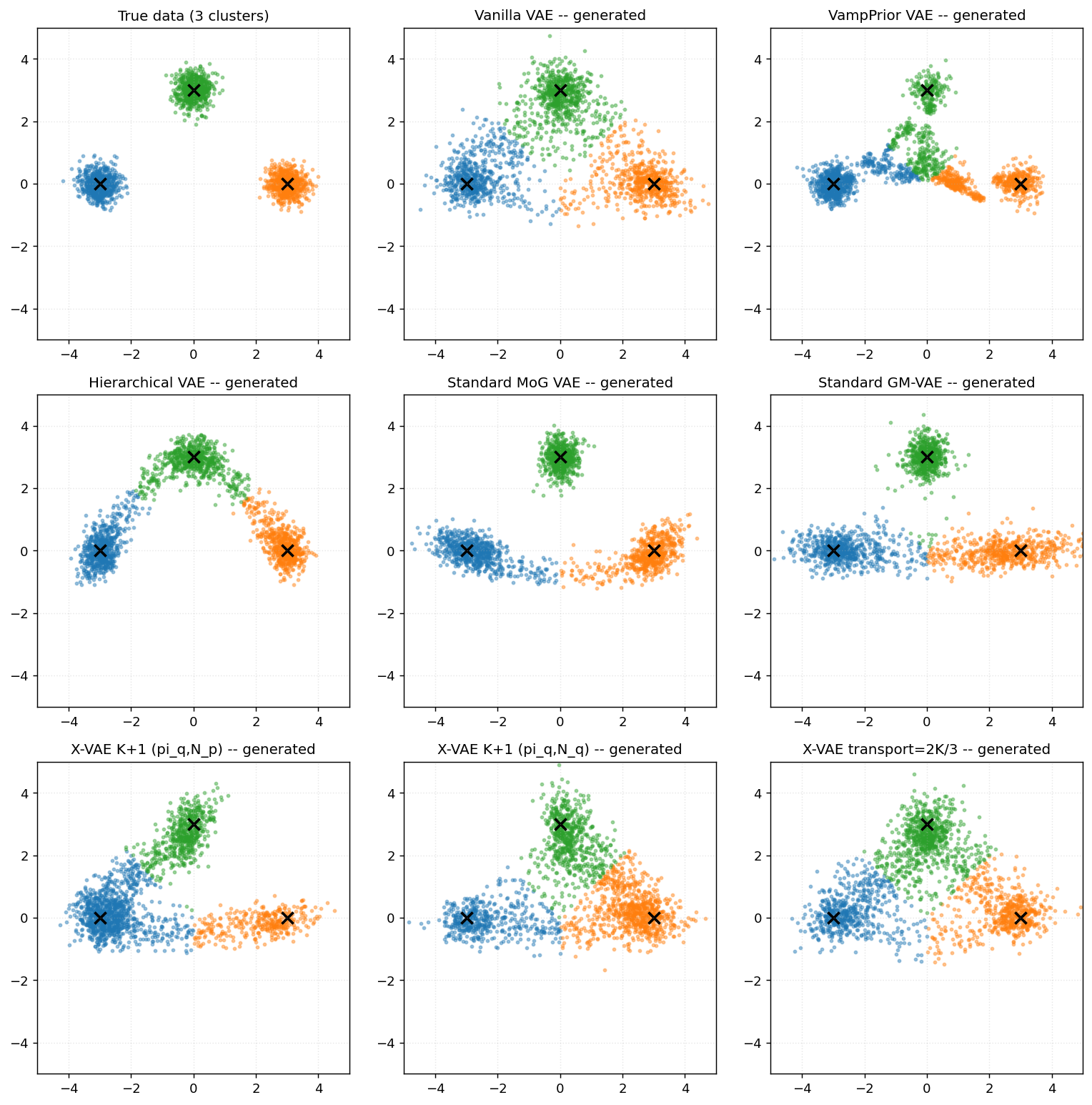}
    \caption{Generated Sample for three clusterings}
    \label{fig:select_cluster_gen}
\end{figure}

No single routing or split is best everywhere, which is itself informative. On CelebA the $K{-}1+1$ routing with the prior-source coupling $(\pi^q,\Ncal^p)$ and a two-thirds transport split are strongest, whereas on MNIST a smaller transport fraction ($\tfrac{1}{3}K$) wins. The spread across our configurations is nonetheless small on each dataset, so the benefit comes from anchoring the latent space to the AE statistics rather than from the precise routing or split: the categorical injects cluster identity (\Cref{prop:routed}) but its exact parameterization is secondary. In practice the transport fraction behaves as a light, dataset-level knob trading reconstruction detail against prior alignment.

For the routed-transport models the KL decomposes into two terms~\Cref{eq:kl}: a categorical term $\KL(\pi^q\,\|\,\pi^p)$ that aligns the batch component-assignment with the prior weights, and a per-component Gaussian term $\sum_k\pi^q_k\,\KL(\Ncal^q_k\,\|\,\Ncal^p_k)$. The first acts as a global consistency term tying the aggregated posterior to the AE-derived mixture, while the second is the familiar per-coordinate regularizer. Empirically the total KL stays moderate (\Cref{tab:loss}), indicating that the posterior remains well aligned to the data-driven prior without collapsing.

Relative to learned-prior approaches---mixture-, flow-, hierarchical-, adversarial-, or diffusion-based priors~\cite{tomczak2018vamp,kingma2017improvingvariationalinferenceinverse,vahdat2021Hierarchy,dilokthanakul2017}---X-VAE estimates its prior once from simple latent statistics with a closed-form KL~\Cref{eq:kl}, keeping training efficient and the likelihood explicit. Because AE embeddings can be learned for any modality, the approach adapts across data types, which makes it attractive for engineering and industrial-design tasks where strong constraints must be enforced while retaining controlled diversity.

For the complexity comparison, \Cref{tab:params} compares trainable parameter counts. The decisive difference lies in the prior: because X-VAE's prior is fixed from latent statistics and stored as non-trainable buffers, it contributes \emph{zero} trainable parameters, whereas the VampPrior pseudo-inputs add $1.6$--$3.9\times10^{5}$, the hierarchical conditional prior $0.7$--$2.8\times10^{5}$, and the learned mixture priors a smaller $1$--$5\times10^{3}$. The single-Gaussian X-VAE therefore carries exactly the trainable footprint of a standard VAE, the smallest in the comparison, while the routed variant adds only a wider encoder head ($3K$ rather than $2K$ outputs). Both are far smaller than the Gumbel--Softmax GM-VAE, whose per-component encoder heads make it roughly $5\times$ larger on CelebA ($23.5$M vs.\ ${\sim}5$M parameters); that much larger model nonetheless yields the worst FID (\Cref{tab:quality}), so X-VAE's improvements do not stem from added capacity. The one cost X-VAE does pay is a one-time autoencoder pretraining ($0.45$M parameters on MNIST, $3.48$M on CelebA, trained before the VAE): the structure that competing methods carry as learned prior weights is instead amortized into this pretraining stage and a closed-form KL, rather than added to the generative model or to per-step inference cost.
 
\begin{table}[t]
\centering\small
\setlength{\tabcolsep}{6pt}
\caption{Trainable parameter counts of the generative model. ``Prior'' isolates the parameters
introduced specifically by the prior---VampPrior pseudo-inputs, the learned mixture parameters of the
MoG/GM-VAE, and the hierarchical conditional-prior network, while X-VAE stores its prior in fixed
buffers and so adds none. Counts include the shared convolutional backbone, which is identical across
all models. X-VAE additionally pretrains an autoencoder ($0.45$M / $3.48$M parameters on MNIST / CelebA),
a one-time first stage not counted here. The GM-VAE's large total comes from its per-component encoder
heads, not from its prior.}
\label{tab:params}
\begin{tabular}{@{}l cc cc@{}}
\toprule
 & \multicolumn{2}{c}{Total trainable ($10^{6}$)} & \multicolumn{2}{c}{Trainable prior ($10^{3}$)} \\
\cmidrule(lr){2-3}\cmidrule(lr){4-5}
Model & MNIST & CelebA & MNIST & CelebA \\
\midrule
Standard VAE        & 0.58 & 4.53  & 0    & 0    \\
VampPrior VAE       & 0.74 & 4.92  & 157  & 393  \\
Hierarchical VAE    & 1.58 & 5.35  & 74   & 280  \\
Standard MoG VAE    & 0.58 & 4.54  & 1.3  & 5.1  \\
GM-VAE              & 2.96 & 23.46 & 1.3  & 5.1  \\
\midrule
X-VAE single-Gaussian & 0.58 & 4.53 & 0 & 0 \\
X-VAE MoG    & 0.71 & 5.58 & 0 & 0 \\
\bottomrule
\end{tabular}
\end{table}

\paragraph{Limitations.} X-VAE relies on the quality of the pretrained autoencoder.
If the AE does not capture meaningful latent structure, the prior may be suboptimal, and because the AE and VAE are trained separately, residual mismatch between their latent spaces can persist~\cite{razavi2019generating}.
From the perspective of architecture, X-VAE introduces an additional dependence on the alignment between AE and VAE latent spaces. Since AE and VAE are trained independently, discrepancies in how they encode data can introduce mismatch between the prior and the learned posterior of VAE~\cite{razavi2019generating}. Ensuring sufficient alignment may require architectural constraints, careful regularization, or additional fine-tuning steps.
For the approach using MoG to fit the latent space, the prior is a diagonal GMM, which can't capture within-mode coordinate correlations, and the framework ties the component count to the latent dimension, so K=64/256 vastly over-specifies the natural number of categories; the EM fit is also fragile in high dimensions.
Thus, future work could explore non-Gaussian priors derived from Autoencoder embeddings, hierarchical extensions, or joint training of the AE and VAE to further enhance performance. Additionally, applying X-VAE to more complex datasets or high-resolution images could reveal further insights into its scalability and robustness~\cite{larsen2016autoencoding, yu2020veridical, james2013islr}.
Furturemore, despite its simplicity, latent scaling introduces several considerations.
The generation-time variance scale $\alpha$ requires care, Large values of $\alpha$ may push latent samples into regions of the latent space that the decoder has not been trained on, potentially producing unrealistic or low-quality outputs, while very small $\alpha$ suppresses diversity.
Scaling the latent noise directly affects the spread of latent vectors input to the decoder. For $\alpha > 1$, the latent samples are further from the prior mean, potentially exploring regions of the latent space that are less frequently visited during training. This can result in outputs with higher diversity or more extreme features. In contrast, $\alpha < 1$ compresses latent samples toward the origin, producing outputs that are more conservative and closer to the “average” learned features. This provides a straightforward mechanism for trading off between novelty and fidelity in generative outputs.
Therefore, similar to $\beta$-VAE, selecting an appropriate $\alpha$ is critical and may require dataset-specific tuning~\cite{higgins2017beta}. 
The $\alpha$ benefits from dataset-specific tuning~\cite{higgins2017beta}.
More fundamentally, a data-dependent prior concentrates latent mass around the training distribution, which can reduce sample diversity and novelty relative to an isotropic prior~\cite{alemi2018fixing,lin2019balancingreconstructionqualityregularisation}.

\section{Conclusions}
\label{sec:conclusion}

In this work, we introduced the eXact-Prior Variational Autoencoder (X-VAE), a generative modeling framework that replaces the conventional isotropic Gaussian prior with a fixed, AE-informed Gaussian mixture and a dimension-wise routed-transport posterior that maps the encoder's per-coordinate Gaussian onto the mixture through a categorical routing of components.
By estimating statistical parameters from the autoencoder embeddings, the proposed method aligns the VAE latent prior with the dataset's empirical structure, thereby reducing the mismatch between the assumed prior and the aggregated posterior distribution~\cite{Bhalodia_2020_ACCV}.

The proposed formulation preserves the standard ELBO, admits a closed-form mixture KL, and binds each latent coordinate to a mixture component to inject cluster structure into the code while incorporating an additional global consistency term that encourages the aggregated posterior to match the autoencoder-derived latent distribution. Furthermore, we introduced a simple yet effective variance-scaling parameter that enables explicit control over the diversity–fidelity trade-off during generation without retraining the model.

Extensive experiments on synthetic clustered data, MNIST, and CelebA datasets demonstrate that X-VAE consistently improves reconstruction fidelity and likelihood metrics compared with standard VAE baselines. In particular, the AE-informed prior leads to better latent manifold alignment and reduces common VAE failure modes when modeling structured or multi-modal data. Empirical results also show that controlled variance scaling can significantly improve generative quality, achieving lower FID scores while maintaining stable latent distributions.


One of the virtues of the present X-VAE is its simplicity. First, all X-VAE (k+1)
formulations have k-dimensional independent Gaussian bases, which are exactly like 
the Vanilla VAE, and we only add an additional GMM model in a single dimension. 
The underlying argument is: if we can match a single Gaussian posterior base to 
that of the prior, the posterior GMM may automatically tune into the prior GMM.
The additional GMM dimension serves only as a guide, so to speak.
Second, we assume that the structure of the posterior GMM has the exact same structure and dimension
as that of the prior GMM. By doing so, it significantly simplifies the KL-divergence expression,
enabling an analytical expression.

Future work will explore several extensions of this framework, including non-Gaussian latent priors derived from autoencoder embeddings, hierarchical or conditional variants of X-VAE, and joint training strategies that further improve alignment between deterministic and probabilistic latent representations~\cite{larsen2016autoencoding}.
We could explore adaptive or learned $\alpha$ schedules, multi-dimensional scaling, or integration with more expressive priors and posterior approximations.
Additionally, systematic evaluation of the impact of $\alpha$ on latent interpolation, disentanglement, and downstream tasks remains an open research direction.
Or investigating hybrid strategies that balance data-dependent priors with controlled latent exploration, for instance, by introducing a tunable scaling factor $\alpha$ on the encoder-derived variance,
\[
z = \mu_\text{AE} + \alpha \cdot \sigma_\text{AE} \odot \epsilon, \quad \epsilon \sim \mathcal{N}(0, I),
\]
combining the benefits of realistic sample alignment with tunable generative diversity. Additionally, integrating these priors into conditional or hierarchical VAE frameworks may further improve both fidelity and flexibility of generated samples.
These directions may enable more expressive and controllable generative models for complex scientific and engineering systems.

\clearpage
\printbibliography

\clearpage
\appendix
\section{Derivation of the Gaussian Mixture-KL Objective}
\label{app:mixkl}

We prove the sampled upper bound~\Cref{eq:kl}. We first establish the general mixture bound and then
specialize to the dimension-wise Gaussian case.

\begin{lemma}[Mixture KL upper bound]
\label{lem:mixkl}
Let $q(z)=\sum_{k=1}^{K}\pi^q_k\,q_k(z)$ and $p(z)=\sum_{k=1}^{K}\pi^p_k\,p_k(z)$ be mixtures with
weights $\pi^q,\pi^p\in\Delta^{K-1}$ and component densities $q_k,p_k$. Then
\begin{equation}
\KL\big(q\,\|\,p\big)\;\le\;\KL(\pi^q\,\|\,\pi^p)+\sum_{k=1}^{K}\pi^q_k\,\KL\big(q_k\,\|\,p_k\big).
\label{eq:mixbound}
\end{equation}
\end{lemma}

\begin{proof}
Introduce the latent component index $k$ and define the joint distributions
$q(z,k)=\pi^q_k\,q_k(z)$ and $p(z,k)=\pi^p_k\,p_k(z)$, whose $z$-marginals are exactly $q$ and $p$.
By the chain rule for the KL divergence,
\begin{equation}
\KL\big(q(z,k)\,\|\,p(z,k)\big)=\KL\big(q(z)\,\|\,p(z)\big)+\E_{q(z)}\!\big[\KL\big(q(k\,|\,z)\,\|\,p(k\,|\,z)\big)\big].
\label{eq:chain1}
\end{equation}
Because the conditional KL term in~\Cref{eq:chain1} is non-negative,
\begin{equation}
\KL\big(q(z)\,\|\,p(z)\big)\;\le\;\KL\big(q(z,k)\,\|\,p(z,k)\big).
\label{eq:marg}
\end{equation}
Applying the chain rule in the other order, $q(z,k)=\pi^q_k\,q_k(z)$ and $p(z,k)=\pi^p_k\,p_k(z)$ give
\begin{align}
\KL\big(q(z,k)\,\|\,p(z,k)\big)
&=\sum_{k=1}^{K}\pi^q_k\!\int q_k(z)\,\log\frac{\pi^q_k\,q_k(z)}{\pi^p_k\,p_k(z)}\,dz \nonumber\\
&=\sum_{k=1}^{K}\pi^q_k\log\frac{\pi^q_k}{\pi^p_k}
 +\sum_{k=1}^{K}\pi^q_k\!\int q_k(z)\log\frac{q_k(z)}{p_k(z)}\,dz \nonumber\\
&=\KL(\pi^q\,\|\,\pi^p)+\sum_{k=1}^{K}\pi^q_k\,\KL\big(q_k\,\|\,p_k\big).
\label{eq:chain2}
\end{align}
Combining~\Cref{eq:marg} and~\Cref{eq:chain2} yields~\Cref{eq:mixbound}.
\end{proof}

\paragraph{Specialization.}
In our model both mixtures factorize over coordinates and, for coordinate $k$, the relevant
components are the Gaussians $q_k=\Ncal^q_k=\Ncal(\mu^q_k,\sigma^{q2}_k)$ and
$p_k=\Ncal^p_k=\Ncal(\mu^p_k,\sigma^{p2}_k)$.
Substituting into~\Cref{eq:mixbound}
gives exactly the training bound~\Cref{eq:kl}. The tighter pairwise estimate of
\cite{hershey2007approximating},
$\KL(q\|p)\le\sum_k\pi^q_k\log\big(\pi^q_k/\sum_j\pi^p_j e^{-\KL(q_k\|p_j)}\big)$, may be substituted
when components overlap; we use the simpler bound~\Cref{eq:mixbound}.

\section{Closed-form Gaussian KL}
\label{app:gausskl}

\begin{lemma}[Univariate Gaussian KL]
\label{lem:gausskl}
For $q=\Ncal(\mu_1,\sigma_1^2)$ and $p=\Ncal(\mu_2,\sigma_2^2)$,
\begin{equation}
\KL(q\,\|\,p)=\log\frac{\sigma_2}{\sigma_1}+\frac{\sigma_1^2+(\mu_1-\mu_2)^2}{2\sigma_2^2}-\frac12.
\label{eq:gauss-app}
\end{equation}
\end{lemma}

\begin{proof}
Writing the log-density ratio,
\begin{equation}
\log\frac{q(z)}{p(z)}=\log\frac{\sigma_2}{\sigma_1}-\frac{(z-\mu_1)^2}{2\sigma_1^2}+\frac{(z-\mu_2)^2}{2\sigma_2^2}.
\end{equation}
Taking the expectation under $q$ and using $\E_q[(z-\mu_1)^2]=\sigma_1^2$ and
$\E_q[(z-\mu_2)^2]=\sigma_1^2+(\mu_1-\mu_2)^2$,
\begin{align}
\KL(q\|p)=\E_q\!\left[\log\frac{q}{p}\right]
&=\log\frac{\sigma_2}{\sigma_1}-\frac{\sigma_1^2}{2\sigma_1^2}+\frac{\sigma_1^2+(\mu_1-\mu_2)^2}{2\sigma_2^2}\nonumber\\
&=\log\frac{\sigma_2}{\sigma_1}+\frac{\sigma_1^2+(\mu_1-\mu_2)^2}{2\sigma_2^2}-\frac12,
\end{align}
which is~\Cref{eq:gauss-app}. Setting $(\mu_1,\sigma_1)=(\mu^q_k,\sigma^q_k)$ and
$(\mu_2,\sigma_2)=(\mu^p_k,\sigma^p_k)$ gives~\Cref{eq:gausskl}.
\end{proof}

\section{Transport Reparameterization and the Routed Marginal}
\label{app:transport}

\begin{lemma}[Matched transport is reparameterization]
\label{lem:transport}
Fix coordinate $k$. If $\epsilon\sim\Ncal(\mu^p_k,\sigma^{p2}_k)$ and
$z_k=T_k(\epsilon)=\mu^q_k+\frac{\sigma^q_k}{\sigma^p_k}(\epsilon-\mu^p_k)$, then
$z_k\sim\Ncal(\mu^q_k,\sigma^{q2}_k)$.
\end{lemma}

\begin{proof}
$z_k$ is an affine function of the Gaussian $\epsilon$, hence Gaussian. Its mean and variance are
\begin{align}
\E[z_k]&=\mu^q_k+\frac{\sigma^q_k}{\sigma^p_k}\big(\E[\epsilon]-\mu^p_k\big)=\mu^q_k,\\
\mathrm{Var}[z_k]&=\Big(\frac{\sigma^q_k}{\sigma^p_k}\Big)^2\mathrm{Var}[\epsilon]
=\Big(\frac{\sigma^q_k}{\sigma^p_k}\Big)^2\sigma^{p2}_k=\sigma^{q2}_k.
\end{align}
Thus $z_k\sim\Ncal(\mu^q_k,\sigma^{q2}_k)$. In particular, when the routing in~\Cref{eq:route} selects
the home component ($c_k=k$, so $M^p_{k,k}=\mu^p_k$ and $S^p_{k,k}=\sigma^p_k$), the categorical leaves
the marginal of $z_k$ unchanged.
\end{proof}

\begin{proposition}[Routed marginal]
\label{prop:routed}
Under the routed coupling, $c_k\sim\Cat(\pi^p)$ and
$\epsilon_k\mid c_k\sim\Ncal(M^p_{c_k,k},S^{p2}_{c_k,k})$, the coordinate
$z_k=\mu^q_k+\frac{\sigma^q_k}{\sigma^p_k}(\epsilon_k-\mu^p_k)$ has the conditional law
\begin{equation}
z_k\mid (c_k=c)\;\sim\;\Ncal\!\Big(\,\underbrace{\mu^q_k+\tfrac{\sigma^q_k}{\sigma^p_k}\big(M^p_{c,k}-\mu^p_k\big)}_{m_c},\;
\underbrace{\big(\tfrac{\sigma^q_k}{\sigma^p_k}S^p_{c,k}\big)^2}_{s_c^2}\Big),
\end{equation}
and the marginal law is the $K$-component Gaussian mixture
$\;z_k\sim\sum_{c=1}^{K}\pi^p_c\,\Ncal(m_c,s_c^2)$.
\end{proposition}

\begin{proof}
Conditioning on $c_k=c$, $\epsilon_k$ is Gaussian and $z_k$ is the affine map $T_k$ of it; as in
\Cref{lem:transport}, $\E[z_k\mid c]=\mu^q_k+\frac{\sigma^q_k}{\sigma^p_k}(M^p_{c,k}-\mu^p_k)=m_c$
and $\mathrm{Var}[z_k\mid c]=(\frac{\sigma^q_k}{\sigma^p_k})^2 S^{p2}_{c,k}=s_c^2$. Marginalizing over
$c_k\sim\Cat(\pi^p)$ gives $p(z_k)=\sum_{c}\pi^p_c\,\Ncal(z_k;m_c,s_c^2)$. For $c=k$ we recover
$m_k=\mu^q_k$, $s_k=\sigma^q_k$ (\Cref{lem:transport}); for $c\neq k$ the component is shifted by
$\frac{\sigma^q_k}{\sigma^p_k}(M^p_{c,k}-\mu^p_k)$ and rescaled by $S^p_{c,k}/\sigma^p_k$, which is how
the routing categorical injects the prior mixture's cross-component geometry into the code.
\end{proof}

\section{Evidence Lower Bound}
\label{app:elbo}

For completeness we recall the ELBO with the fixed mixture prior. For any approximate posterior
$q_\phi(z|x)$,
\begin{align}
\log p_\theta(x)
&=\log\int p_\theta(x|z)\,p(z)\,dz
=\log\,\E_{q_\phi(z|x)}\!\left[\frac{p_\theta(x|z)\,p(z)}{q_\phi(z|x)}\right]\nonumber\\
&\ge\E_{q_\phi(z|x)}\!\left[\log\frac{p_\theta(x|z)\,p(z)}{q_\phi(z|x)}\right]
=\E_{q_\phi(z|x)}[\log p_\theta(x|z)]-\KL\big(q_\phi(z|x)\,\|\,p(z)\big)=\mathcal{L}(x),
\end{align}
where the inequality is Jensen's. Substituting the mixture KL bound of \Cref{lem:mixkl} for the
KL term gives a (looser) tractable lower bound on $\log p_\theta(x)$ that is maximized during training;
the reparameterized transport~\Cref{eq:transport} provides low-variance gradients of the first term
with respect to $\phi$.

\section{Detailed KL Divergence Derivation}
\label{kl_derive}
Consider two $d$-dimensional diagonal Gaussian distributions:

\begin{align}
q(z) &= \mathcal{N}(\mathbf{z} \mid \bm{\mu}_q, \mathrm{diag}(\bm{\sigma}_q^2)), \\
p(z) &= \mathcal{N}(\mathbf{z} \mid \bm{\mu}_{p}, \mathrm{diag}(\bm{\sigma}_{p}^2)).
\end{align}

where $p(z)$ is determined by autoencoder. The Kullback–Leibler divergence is defined as:

\begin{equation}
D_{\mathrm{KL}}(q \,\Vert\, p) = \int q(z) \log \frac{q(z)}{p(z)} \, dz.
\end{equation}


The probability density function of a $d$-dimensional diagonal Gaussian is:

\begin{equation}
\mathcal{N}(\mathbf{z} \mid \bm{\mu}, \mathrm{diag}(\bm{\sigma}^2)) = 
\frac{1}{(2 \pi)^{d/2} \prod_{j=1}^d \sigma_j} 
\exp\Bigg[-\frac{1}{2} \sum_{j=1}^d \frac{(z_j - \mu_j)^2}{\sigma_j^2} \Bigg].
\end{equation}

Hence:

\begin{align}
q(z) &= \frac{1}{(2 \pi)^{d/2} \prod_{j=1}^d \sigma_{q,j}} 
\exp\Bigg[-\frac{1}{2} \sum_{j=1}^d \frac{(z_j - \mu_{q,j})^2}{\sigma_{q,j}^2} \Bigg], \\
p(z) &= \frac{1}{(2 \pi)^{d/2} \prod_{j=1}^d \sigma_{p,j}} 
\exp\Bigg[-\frac{1}{2} \sum_{j=1}^d \frac{(z_j - \mu_{p,j})^2}{\sigma_{p,j}^2} \Bigg].
\end{align}


\begin{align}
\log \frac{q(z)}{p(z)} &= \log q(z) - \log p(z) \\
&= \Big[-\frac{d}{2}\log(2\pi) - \sum_{j=1}^d \log \sigma_{q,j} - \frac{1}{2} \sum_{j=1}^d \frac{(z_j - \mu_{q,j})^2}{\sigma_{q,j}^2} \Big] \\
& \quad - \Big[-\frac{d}{2}\log(2\pi) - \sum_{j=1}^d \log \sigma_{p,j} - \frac{1}{2} \sum_{j=1}^d \frac{(z_j - \mu_{p,j})^2}{\sigma_{p,j}^2} \Big] \\
&= \sum_{j=1}^d \Big[ \log \frac{\sigma_{p,j}}{\sigma_{q,j}} + \frac{1}{2} \Big(\frac{(z_j - \mu_{p,j})^2}{\sigma_{p,j}^2} - \frac{(z_j - \mu_{q,j})^2}{\sigma_{q,j}^2}\Big) \Big].
\end{align}


\begin{equation}
D_{\mathrm{KL}}(q \,\Vert\, p) = \mathbb{E}_{z \sim q} \Big[ \sum_{j=1}^d \log \frac{\sigma_{p,j}}{\sigma_{q,j}} + \frac{1}{2} \Big(\frac{(z_j - \mu_{p,j})^2}{\sigma_{p,j}^2} - \frac{(z_j - \mu_{q,j})^2}{\sigma_{q,j}^2} \Big) \Big].
\end{equation}

The first term is constant with respect to $z$, and for the second term, note that:

\begin{equation}
\mathbb{E}_{z \sim q}[(z_j - \mu_{q,j})^2] = \sigma_{q,j}^2,
\end{equation}

\begin{equation}
\mathbb{E}_{z \sim q}[(z_j - \mu_{p,j})^2] = \sigma_{q,j}^2 + (\mu_{q,j} - \mu_{p,j})^2.
\end{equation}


\begin{align}
D_{\mathrm{KL}}(q \,\Vert\, p) &= \sum_{j=1}^d \Big[ \log \frac{\sigma_{p,j}}{\sigma_{q,j}} + \frac{1}{2} \Big( \frac{\sigma_{q,j}^2 + (\mu_{q,j} - \mu_{p,j})^2}{\sigma_{p,j}^2} - \frac{\sigma_{q,j}^2}{\sigma_{q,j}^2} \Big) \Big] \\
&= \sum_{j=1}^d \Big[ \log \frac{\sigma_{p,j}}{\sigma_{q,j}} + \frac{\sigma_{q,j}^2 + (\mu_{q,j} - \mu_{p,j})^2}{2 \sigma_{p,j}^2} - \frac{1}{2} \Big].
\end{align}


while

\begin{equation}
\log \frac{\sigma_{p,j}}{\sigma_{q,j}} = \frac{1}{2} \log \frac{\sigma_{p,j}^2}{\sigma_{q,j}^2},
\end{equation}

giving the final closed form of KL Divergence:

\begin{equation}
D_{\mathrm{KL}}(q \,\Vert\, p) = \frac{1}{2} \sum_{j=1}^d \left[ \log \frac{\sigma_{p,j}^2}{\sigma_{q,j}^2} + \frac{\sigma_{q,j}^2 + (\mu_{q,j} - \mu_{p,j})^2}{\sigma_{p,j}^2} - 1 \right]
\end{equation}

\section{Controlled Latent Sampling: Mathematical Justification}

In X-VAE, the latent prior is derived from a pretrained autoencoder (AE):

\begin{equation}
z \sim p_{\text{AE}}(z) = \mathcal{N}(\mu_{\text{AE}}, \mathrm{diag}(\sigma_{\text{AE}}^2)).
\end{equation}


A standard generative sampling without scaling uses:

\begin{equation}
z = \mu_{\text{AE}} + \sigma_{\text{AE}} \odot \epsilon, \quad \epsilon \sim \mathcal{N}(0, I),
\end{equation}

so that $\mathbb{E}[z] = \mu_{\text{AE}}$ and $\mathrm{Var}[z] = \sigma_{\text{AE}}^2$.


We introduce a scale factor $s > 0$ for controlled exploration:

\begin{equation}
z_s = \mu_{\text{AE}} + \alpha \cdot \sigma_{\text{AE}} \odot \epsilon, \quad \epsilon \sim \mathcal{N}(0, I).
\end{equation}

Then, by standard properties of Gaussian random variables:

\begin{align}
\mathbb{E}[z_s] &= \mu_{\text{AE}}, \\
\mathrm{Var}[z_s] &= \mathrm{Var}[\alpha \cdot \sigma_{\text{AE}} \odot \epsilon] 
= \alpha^2 \, \sigma_{\text{AE}}^2.
\end{align}


Let $q_\phi(z|x)$ be the approximate posterior learned by the X-VAE. 
During generation, we desire $z$ to cover regions of the latent space that correspond 
to meaningful outputs. Scaling the variance modifies the prior's spread:

\begin{itemize}
    \item If $\alpha > 1$, the latent space is explored more broadly, potentially increasing sample diversity.  
    \item If $\alpha < 1$, the latent space is concentrated near the mean, reducing variance in generated samples.  
\end{itemize}

Mathematically, $z_s \sim \mathcal{N}(\mu_{\text{AE}}, \alpha^2 \, \mathrm{diag}(\sigma_{\text{AE}}^2))$, 
so the generation distribution remains Gaussian but with **controllable variance**.


The Fréchet Inception Distance (FID) between generated samples and data is sensitive to the 
mean and covariance of the latent codes. Denote:

\begin{align}
\mu_g &= \mathbb{E}[z_s] = \mu_{\text{AE}}, \\
\Sigma_g &= \mathrm{Var}[z_s] = \alpha^2 \, \mathrm{diag}(\sigma_{\text{AE}}^2), \\
\Sigma_d &= \mathrm{Var}[z_{\text{data}}].
\end{align}

FID is roughly minimized when the generated latent mean and covariance match the true latent distribution. 
Introducing $\alpha \neq 1$ adjusts $\Sigma_g$ toward $\Sigma_d$, potentially reducing FID, 
as observed empirically.


Scaling the latent standard deviation:

\begin{equation}
z_s = \mu_{\text{AE}} + \alpha \cdot \sigma_{\text{AE}} \odot \epsilon
\end{equation}

\begin{itemize}
    \item Preserves the Gaussian nature of the AE prior  
    \item Allows controlled variance adjustment  
    \item Matches latent covariance to data distribution more closely  
    \item Improves sample quality and FID in practice
\end{itemize}

This provides a **mathematical justification** for controlled latent sampling in X-VAE.

\section{Theoretical insight: Improved manifold alignment}

Let \( M_x \subset \mathbb{R}^d \) denote the data manifold and \( M_y \subset \mathbb{R}^k \) the latent manifold. In a VAE with an isotropic prior, \( M_y \) is regularized to align with a spherical Gaussian. However, in X-VAE, since \( P_{\text{X-VAE}}(z) \) inherits curvature from the AE's latent statistics, we effectively re-shape the latent regularizer:

\[
\text{Cov}(P_{\text{X-VAE}}(z)) = f_{\Sigma}(\Sigma_{\text{AE}})
\]

where the latent space is no longer isotropic but aligned with the empirical curvature of the data.

Formally, the pushforward of the AE encoder \( f_{\psi} \) defines a map:

\[
z = f_{\psi}(x): x \in M_x
\]

If \( f_{\psi} \) is locally diffeomorphic (smooth and invertible in a neighborhood), then \( M_z = f_{\psi}(M_x) \) is a differentiable submanifold. By setting the VAE prior to match this manifold's empirical distribution \( P_{\text{AE}}(z) \), we minimize the mismatch between the learned posterior \( q_{\phi}(z | x) \) and the true latent manifold.

\clearpage
\section{Implementation Details}
\label{imp_detail}



\begin{table}[!htb]
\centering
\caption{Implementation Details for MNIST and CelebA}
\label{tab:implement_details}
\small 
\begin{tabular}{@{} l c|c @{}}
\toprule
\textbf{Parameter} & \textbf{MNIST} & \textbf{CelebA} \\ 
\midrule
Image size & $1 \times 28 \times 28$ & $3 \times 64 \times 64$ \\
Encoder channels (CNN~\cite{lecun1989}) & 1--32--64--128 & 3--32--64--128--256 \\
Latent dimension ($K$) & 64 & 256 \\
VampPrior pseudo-inputs & 200 & 32 \\
Batch size & 128 & 128 \\
Epochs (AE / VAE) & 10 / 15 & 10 / 15 \\
\hline
Training loss & \multicolumn{2}{c}{MSE (Sum Reduction)} \\
Optimizer & \multicolumn{2}{c}{Adam~\cite{kingma2017adammethodstochasticoptimization}} \\
Learning rate & \multicolumn{2}{c}{$1 \times 10^{-3}$} \\
Random State & \multicolumn{2}{c}{$42$}\\
\bottomrule
\end{tabular}
\end{table}

\begin{table}[!htb]
\centering
\caption{FID Evaluation Protocol}
\begin{tabular}{ll}
\hline
\textbf{Component} & \textbf{Setting} \\
    \hline
    Real samples & 1,000 images \\
    Fake samples & 2,000 images \\
    Latent sampling & $\mathcal{N}(\mu_{AE}, (\alpha \sigma_{AE})^2)$ \\
    Batch size (generation) & 128 \\
    Image format & PNG \\
    Real image normalization & $[0,1]$ \\
    Fake image normalization & Decoder output \\
    FID implementation & PyTorch-FID \\
    Feature extractor & Inception-V3 \\
    FID batch size & 50  \\
    \hline
\end{tabular}
\label{tab:fid_protocol}
\end{table}

\begin{table}[!htb] 
\centering 
\caption{Alienware Aurora R15 Hardware Specifications} 
\label{tab:alienware_r15_specs} 
\begin{tabular}{c | c} 
\toprule 
\textbf{Component} & \textbf{Specification} \\ 
\midrule 
\textbf{Processor (CPU)} & 13th Gen Intel Core i9-13900F \\ 
\textbf{CPU Cores} & 24-Core, 68MB Cache \\ 
\textbf{Memory (RAM)} & 32GB (2 x 16GB) DDR5 @ 4800MHz \\ 
\textbf{Graphics (GPU)} & NVIDIA GeForce RTX 4090\\ 
\textbf{GPU Memory (RAM)} & 24GB GDDR6X \\ 
\textbf{Storage (Typical)} & 1TB NVMe M.2 PCIe SSD\\ 
\textbf{Motherboard} & Proprietary Intel Z790 Chipset \\ 
\textbf{Operating System} & Ubuntu 24.04 LTS (x86\_64) \\ 
\textbf{Torch Version} & 2.9.0 \\ 
\textbf{CUDA Version} & 12.6 \\ 
\bottomrule 
\end{tabular} 
\end{table}

\section{Full Per-Model Results}
\label{app:fullresults}

\Cref{tab:quality-full}, \Cref{tab:loss-full}, and\Cref{tab:loss-clusters-full} give the complete per-model results summarized in
\Cref{sec:experiments}: all baselines and every configuration of our method (the single-Gaussian
X-VAE, the interleaved routed-transport model, its four $K{+}1$ routing ablations, and the
transport-split sweep). Best per column in bold; the clustered data ($K{=}3$) is summarized by the
training objective only, since Inception-based FID/IS do not apply to its two-dimensional samples.

\begin{table}[!htb]
\centering\small
\setlength{\tabcolsep}{4.5pt}
\caption{Full results: reconstruction error (MSE), FID, and Inception Score (IS) on MNIST and CelebA.
Lower is better for MSE and FID, higher for IS; best per column in bold.}
\label{tab:quality-full}
\begin{tabular}{@{}l ccc ccc@{}}
\toprule
 & \multicolumn{3}{c}{MNIST ($K{=}64$)} & \multicolumn{3}{c}{CelebA ($K{=}256$)} \\
\cmidrule(lr){2-4}\cmidrule(lr){5-7}
Model & Rec$\downarrow$ & FID$\downarrow$ & IS$\uparrow$ & Rec$\downarrow$ & FID$\downarrow$ & IS$\uparrow$ \\
\midrule
Standard VAE        & 11.24 & 49.67 & $2.382{\pm}.068$ & 87.42 & 86.34 & $1.816{\pm}.052$ \\
VampPrior VAE       & 12.12 & 47.76 & $2.460{\pm}.093$ & 90.27 & 82.80 & $\mathbf{1.858}{\pm}.047$ \\
Hierarchical VAE    & 11.89 & 46.64 & $2.411{\pm}.065$ & 104.60 & 87.09 & $1.806{\pm}.073$ \\
Standard MoG VAE    & 12.88 & 58.27 & $\mathbf{2.631}{\pm}.101$ & 92.63 & 87.49 & $1.795{\pm}.034$ \\
GM-VAE              & $\mathbf{4.98}$ & 79.08 & $2.062{\pm}.063$ & $\mathbf{74.99}$ & 89.14 & $1.781{\pm}.037$ \\
\midrule
X-VAE single-Gaussian            & 12.18 & 50.17 & $2.423{\pm}.089$ & 88.86 & 84.45 & $1.807{\pm}.053$ \\
Ours, interleaved                & 11.85 & 47.04 & $2.435{\pm}.119$ & 91.64 & 85.46 & $1.823{\pm}.066$ \\
\quad $K{-}1+1$ ($\pi^p,\Ncal^p$)  & 12.30 & 48.42 & $2.391{\pm}.113$ & 88.09 & 84.79 & $1.774{\pm}.063$ \\
\quad $K{-}1+1$ ($\pi^q,\Ncal^p$)  & 11.89 & 47.41 & $2.402{\pm}.071$ & 87.99 & $\mathbf{80.40}$ & $1.772{\pm}.060$ \\
\quad $K{-}1+1$ ($\pi^p,\Ncal^q$)  & 12.12 & 48.13 & $2.418{\pm}.092$ & 92.79 & 86.55 & $1.818{\pm}.072$ \\
\quad $K{-}1+1$ ($\pi^q,\Ncal^q$)  & 11.60 & 47.21 & $2.411{\pm}.079$ & 88.40 & 83.11 & $1.797{\pm}.032$ \\
\quad transport $K$ (full)       & 12.26 & 49.04 & $2.360{\pm}.083$ & 87.54 & 85.23 & $1.746{\pm}.029$ \\
\quad transport $\tfrac{2}{3}K$  & 11.65 & 47.54 & $2.382{\pm}.106$ & 88.24 & 81.90 & $1.807{\pm}.070$ \\
\quad transport $\tfrac{1}{3}K$  & 12.06 & $\mathbf{46.01}$ & $2.370{\pm}.107$ & 87.66 & 89.07 & $1.848{\pm}.066$ \\
\bottomrule
\end{tabular}
\end{table}

\begin{table}[!htb]
\centering\footnotesize
\setlength{\tabcolsep}{4pt}
\caption{Full results: final-epoch train/test objective with reconstruction and KL terms on MNIST and
CelebA. Lowest test objective per dataset in bold; magnitudes are not comparable across datasets.}
\label{tab:loss-full}
\begin{tabular}{@{}l cccc cccc@{}}
\toprule
 & \multicolumn{4}{c}{MNIST ($K{=}64$)} & \multicolumn{4}{c}{CelebA ($K{=}256$)} \\
\cmidrule(lr){2-5}\cmidrule(lr){6-9}
Model & train & test & recon & KL & train & test & recon & KL \\
\midrule
Standard VAE        & 28.96 & 28.85 & 16.67 & 12.29 & 167.34 & 165.19 & 111.75 & 55.59 \\
VampPrior VAE       & 28.59 & 28.30 & 17.51 & 11.08 & 166.09 & 166.22 & 111.51 & 54.57 \\
Hierarchical VAE    & 28.31 & $\mathbf{28.10}$ & 16.64 & 11.67 & 173.22 & 170.79 & 124.64 & 48.58 \\
Standard MoG VAE    & 29.50 & 29.24 & 17.70 & 11.80 & 167.87 & 167.96 & 112.89 & 54.98 \\
GM-VAE              & 69.37 & 69.17 & 10.03 & 59.33 & 231.67 & 232.07 & 103.47 & 128.21 \\
\midrule
X-VAE single-Gaussian            & 29.16 & 29.00 & 17.29 & 11.86 & 166.68 & 165.51 & 110.90 & 55.77 \\
Ours, interleaved                & 29.02 & 28.75 & 16.91 & 12.11 & 166.74 & 167.44 & 110.85 & 55.89 \\
\quad $K{-}1+1$ ($\pi^p,\Ncal^p$)  & 29.11 & 28.90 & 17.24 & 11.87 & 166.49 & $\mathbf{164.19}$ & 110.68 & 55.80 \\
\quad $K{-}1+1$ ($\pi^q,\Ncal^p$)  & 28.82 & 28.61 & 16.70 & 12.13 & 166.64 & 164.82 & 110.61 & 56.03 \\
\quad $K{-}1+1$ ($\pi^p,\Ncal^q$)  & 28.89 & 28.69 & 17.06 & 11.83 & 166.73 & 167.15 & 111.11 & 55.62 \\
\quad $K{-}1+1$ ($\pi^q,\Ncal^q$)  & 28.91 & 28.56 & 16.88 & 12.04 & 167.88 & 164.65 & 111.95 & 55.93 \\
\quad transport $K$ (full)       & 29.06 & 28.97 & 17.09 & 11.97 & 167.03 & 165.31 & 111.30 & 55.73 \\
\quad transport $\tfrac{2}{3}K$  & 29.08 & 28.75 & 16.98 & 12.10 & 166.59 & 165.42 & 110.58 & 56.01 \\
\quad transport $\tfrac{1}{3}K$  & 29.18   & 29.16   & 17.47   & 11.71   & 166.82 & 164.19 & 110.86 & 55.96 \\
\bottomrule
\end{tabular}
\end{table}

\begin{table}[!htb]
\centering\footnotesize
\setlength{\tabcolsep}{5pt}
\caption{Full results on the clustered data ($K{=}3$): final-epoch train/test objective with
reconstruction and KL terms. GM-VAE's objective reflects a degenerate (near-zero/negative) sampled KL.
The interleaved model and transport $\tfrac{1}{3}K$ both reduce to a single transport coordinate at
$K{=}3$ and reconstruct markedly worse than configurations with a larger transport fraction.}
\label{tab:loss-clusters-full}
\begin{tabular}{@{}l cccc@{}}
\toprule
Model & train & test & recon & KL \\
\midrule
Standard VAE        & 0.048 & 0.048 & 0.010 & 0.038 \\
VampPrior VAE       & 0.049 & 0.051 & 0.011 & 0.038 \\
Hierarchical VAE    & 0.047 & 0.047 & 0.011 & 0.036 \\
Standard MoG VAE    & 0.047 & 0.046 & 0.011 & 0.036 \\
GM-VAE              & 0.008 & 0.007 & 0.010 & $-0.002$ \\
\midrule
X-VAE single-Gaussian            & 0.049 & 0.048 & 0.011 & 0.038 \\
Ours, interleaved                & 0.115 & 0.120 & 0.086 & 0.029 \\
\quad $K{-}1+1$ ($\pi^p,\Ncal^p$)  & 0.053 & 0.053 & 0.013 & 0.041 \\
\quad $K{-}1+1$ ($\pi^q,\Ncal^p$)  & 0.053 & 0.051 & 0.011 & 0.042 \\
\quad $K{-}1+1$ ($\pi^p,\Ncal^q$)  & 0.052 & 0.051 & 0.011 & 0.041 \\
\quad $K{-}1+1$ ($\pi^q,\Ncal^q$)  & 0.056 & 0.057 & 0.015 & 0.040 \\
\quad transport $K$ (full)       & 0.053 & 0.054 & 0.012 & 0.041 \\
\quad transport $\tfrac{2}{3}K$  & 0.055 & 0.051 & 0.013 & 0.041 \\
\quad transport $\tfrac{1}{3}K$  & 0.119 & 0.121 & 0.091 & 0.028 \\
\bottomrule
\end{tabular}
\end{table}

\clearpage
\section{Reconstructions for all models}
\label{app:full_recon}
\begin{figure}[!htb]
\centering
\includegraphics[width=0.5\linewidth]{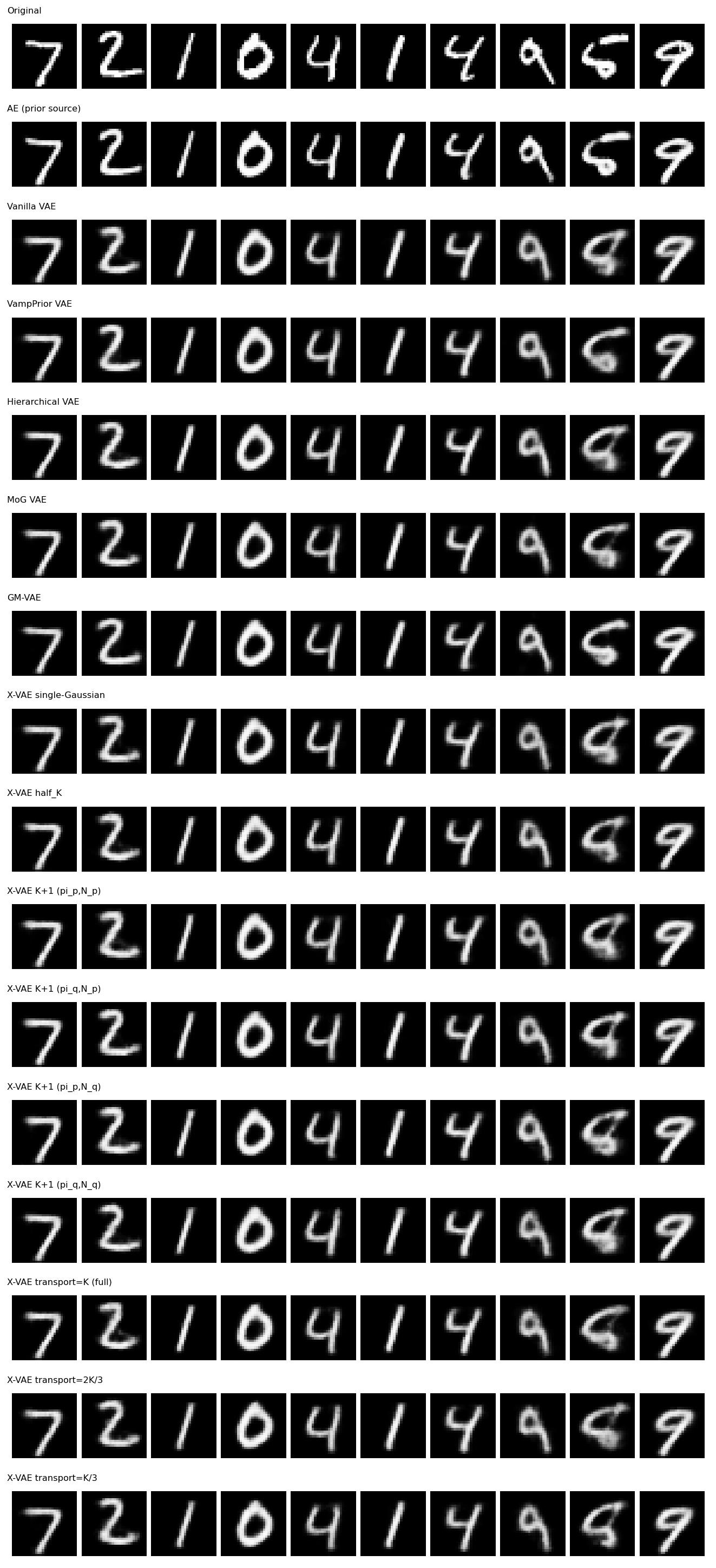}
\caption{MNIST reconstructions}
\label{fig:recon-mnist}
\end{figure}

\begin{figure}[!htb]
\centering
\includegraphics[width=0.8\linewidth]{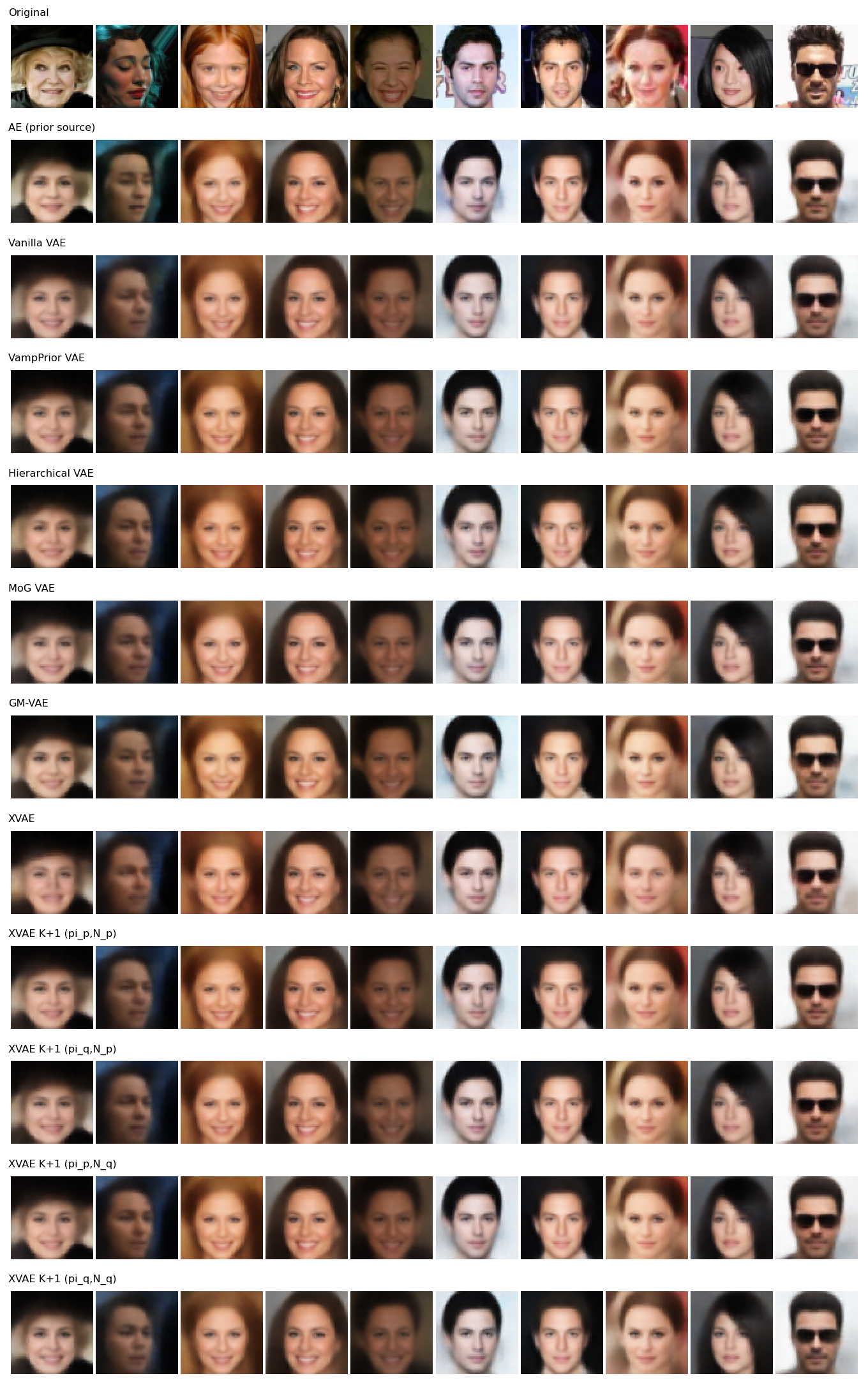}
\caption{CelebA reconstructions}
\label{fig:recon-celeba}
\end{figure}

\begin{figure}[!htb]
\centering
\includegraphics[width=0.6\linewidth]{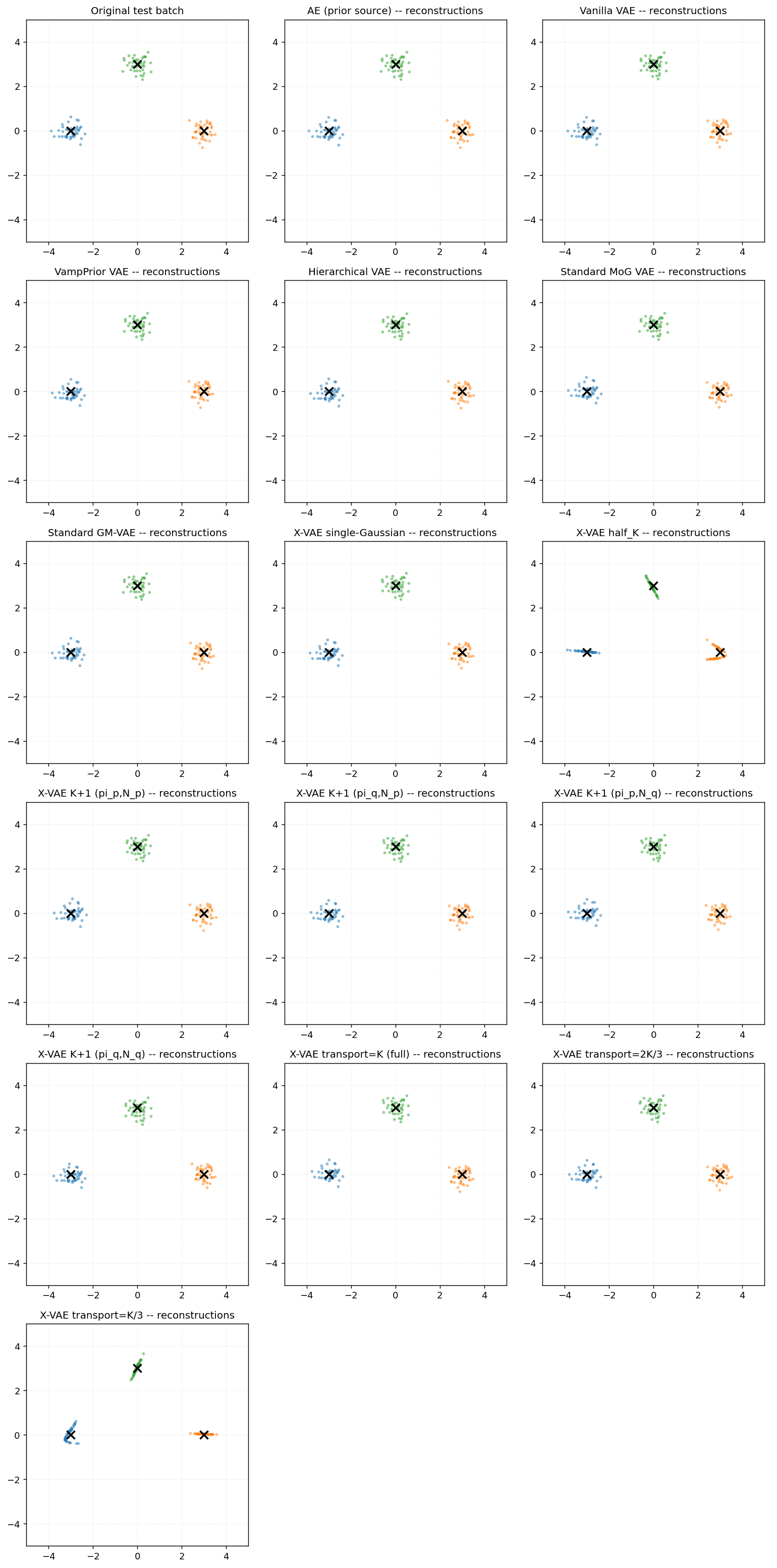}
\caption{Clustering reconstructions}
\label{fig:recon-cluster}
\end{figure}

\clearpage
\section{Generations for all models}

\begin{figure}[!htb]
\centering
\includegraphics[width=\linewidth]{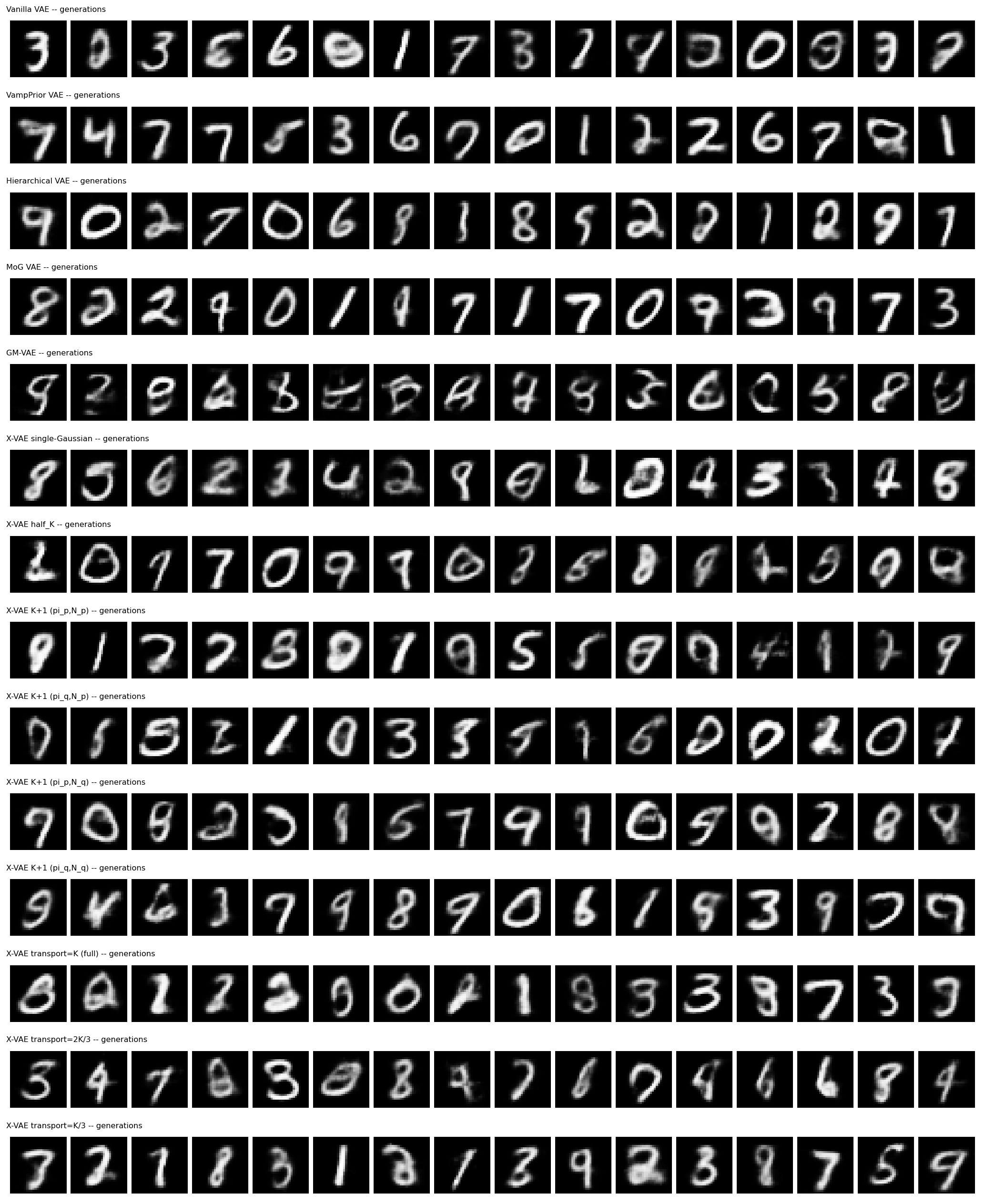}
\caption{MNIST generations}
\label{fig:gen-mnist}
\end{figure}

\begin{figure}[!htb]
\includegraphics[width=\linewidth]{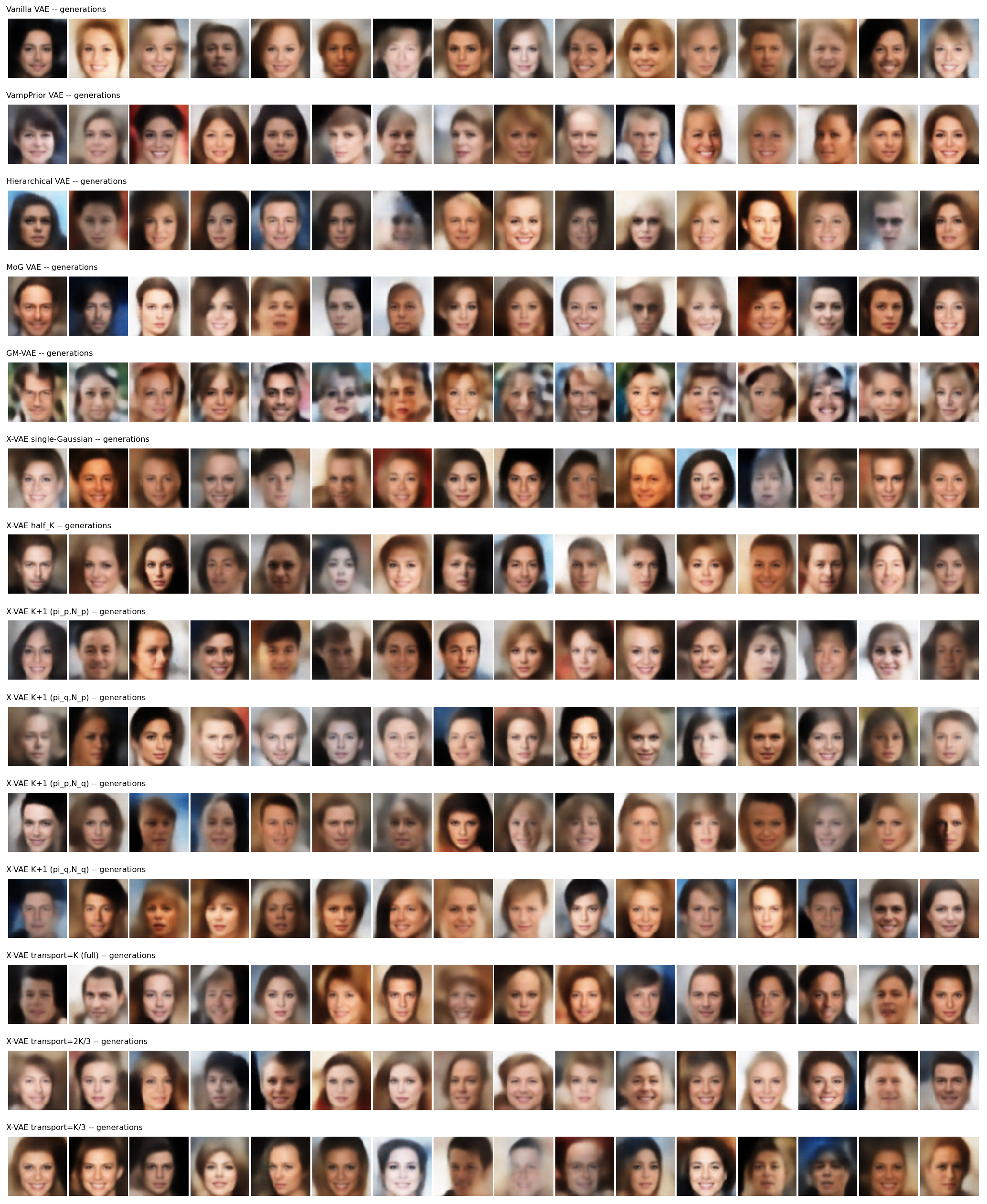}
\caption{CelebA generations}
\label{fig:gen-celeba}
\end{figure}

\begin{figure}[!htb]
\centering
\includegraphics[width=0.7\linewidth]{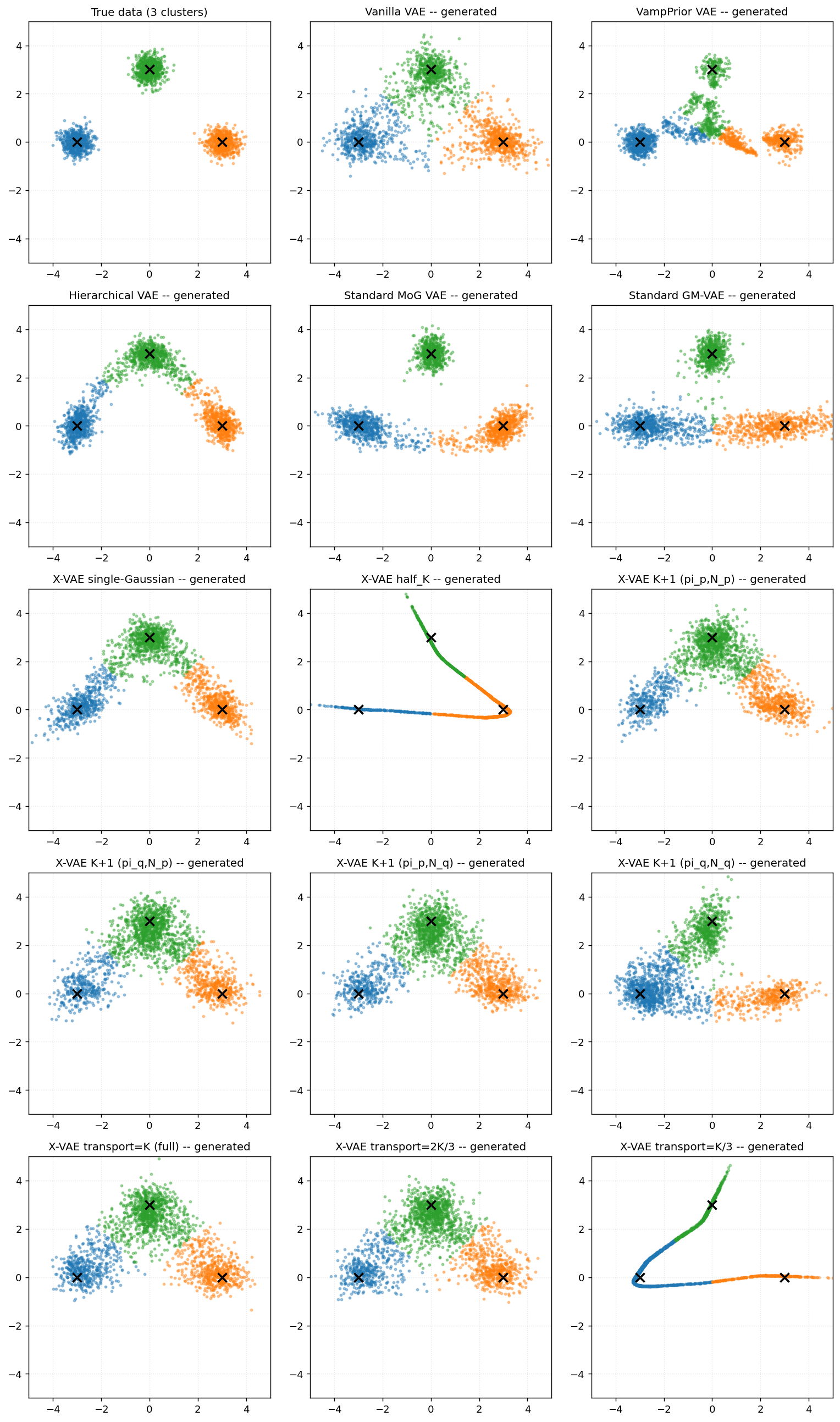}
\caption{Clustered data: samples generated by our method (preserving the three modes) versus a standard
VAE.}
\label{fig:gen-clusters}
\end{figure}
\end{document}